\documentclass{article}

\usepackage{microtype}
\usepackage{graphicx}
\usepackage{subcaption}
\usepackage{booktabs} 
\usepackage{multirow}
\usepackage[table]{xcolor}
\usepackage{pifont}
\usepackage{makecell}

\usepackage{booktabs, multirow, colortbl}
\usepackage{threeparttable} 
\usepackage{xcolor}   
\usepackage{colortbl}  
\usepackage{multirow} 
\usepackage{booktabs} 
\usepackage{adjustbox}  
\usepackage{booktabs} 
\usepackage{tabularx}
\usepackage{booktabs}
\usepackage{multirow}
\usepackage{xcolor}
\usepackage{float}
\usepackage{fontawesome}
\definecolor{softpink}{RGB}{252,240,242}
\definecolor{mf0670}{HTML}{F8E6A0}
\definecolor{lightblue}{RGB}{135, 185, 235}
\usepackage{float}

\usepackage{hyperref}
\definecolor{rowshade}{HTML}{F1F6EC}

\usepackage{colortbl} 
\definecolor{hici}{RGB}{239,247,242}

\usepackage[accepted]{icml2026}

\usepackage{amsmath}
\usepackage{amssymb}
\usepackage{mathtools}
\usepackage{amsthm}

\usepackage[capitalize,noabbrev]{cleveref}

\theoremstyle{plain}
\newtheorem{theorem}{Theorem}[section]

\theoremstyle{definition}

\theoremstyle{remark}
\newtheorem{remark}[theorem]{Remark}

\usepackage[textsize=tiny]{todonotes}

\icmltitlerunning{HiCI: Hierarchical Construction–Integration for Long-Context Attention}
\usepackage{enumitem}

\begin{document}

\twocolumn[
  \icmltitle{HiCI: Hierarchical Construction–Integration for Long-Context Attention}

  \icmlsetsymbol{equal}{*}

  \begin{icmlauthorlist} 
    \icmlauthor{Xiangyu Zeng}{syd} 
    \icmlauthor{Qi Xu}{whu} 
    \icmlauthor{Yunke Wang}{syd} 
    \icmlauthor{Chang Xu}{syd}  
  \end{icmlauthorlist}   
  
  \icmlaffiliation{syd}{School of Computer Science, University of Sydney, Sydney, Australia}      
  \icmlaffiliation{whu}{Westlake University, Hangzhou, China}  
  \icmlcorrespondingauthor{Chang Xu}{c.xu@sydney.edu.au}
   \begin{center}  
    \small
    \faGithub\enspace\href{https://github.com/zengxyyu/HiCI.git}{\texttt{github.com/zengxyyu/HiCI}}         
  \end{center}

  \icmlkeywords{Machine Learning, ICML, Long-context language modeling}

  \vskip 0.3in
]

\printAffiliationsAndNotice{}  

\begin{abstract}
Long-context language modeling is commonly framed as a scalability challenge of token-level attention, yet local-to-global information structuring remains largely implicit in existing approaches. Drawing on cognitive theories of discourse comprehension, we propose \textbf{HiCI} (Hierarchical Construction--Integration), a hierarchical attention module that constructs segment-level representations, integrates them into a shared global context, and broadcasts both to condition segment-level attention. We validate HiCI through parameter-efficient adaptation of LLaMA-2 (7B and 13B) and Qwen3-8B with only $\sim$4--5\% additional parameters, extending context to 100K/64K tokens for LLaMA-2-7B/13B and to 48K tokens for Qwen3-8B.
Across language modeling, retrieval, and instruction-following benchmarks, HiCI yields consistent improvements over strong baselines, including matching proprietary models on topic retrieval and surpassing GPT-3.5-Turbo-16K on code comprehension. 
These results demonstrate the effectiveness of explicit hierarchical structuring as an inductive bias for long-context modeling.  

\end{abstract}

\section{Introduction}
\label{sec:intro}
Large language models (LLMs) have achieved remarkable success across a wide range of natural language tasks, yet their ability to process long sequences remains fundamentally constrained by limited context windows \citep{vaswani2017attention,brown2020language, li2025identify, wang2025position, wang2026sibyl}.
Long-context modeling poses two fundamental challenges: (1) \textit{efficiency}—the quadratic complexity of self-attention leads to prohibitive computational and memory costs as sequence length increases; and (2) \textit{effectiveness}—the ability to accept longer inputs does not necessarily yield reliable modeling of long-range dependencies \citep{DBLP:journals/corr/abs-2404-06654, liu2024lost}
. Reconciling these two requirements has emerged as a central challenge in long-context language modeling.

Recent work has progressed along two complementary lines. The first pursues \textbf{positional length generalization}:
techniques such as PI, YaRN, and PoSE~\citep{chen2023extending, peng2024yarn, zhu2024pose} extend the usable context window by interpolating, rescaling, or simulating position indices, yet leave the attention operator—and its $\mathcal{O}(n^2)$ complexity—unchanged.
The second focuses on \textbf{attention efficiency and architectural scalability}, comprising two broad families. Sparse and grouped attention~\citep{beltagy2020longformer, zaheer2020bigbird, chen2024longlora} reduces cost by restricting token connectivity, with global interactions approximated via global tokens
or layer-wise multi-hop mixing from shifted grouping. Recurrent and memory-augmented architectures~\citep{dai2019transformerxl, bulatov2022recurrent, munkhdalai2024leave, he2025hmt}
model cross-segment dependencies through compressed state propagation, but sequential processing limits parallelism and long-range information may be attenuated through the compression bottleneck.
While effective for length generalization or efficiency, these approaches offer limited inductive bias for explicitly organizing long-context information into a local-to-global hierarchy that guides attention.

Cognitive theories of discourse comprehension offer a principled lens on this limitation.
The Construction-Integration model~\citep{kintsch1988role,kintsch1998comprehension} characterizes text understanding as a hierarchical process in which local representations are first constructed from input segments and subsequently integrated—via constraint satisfaction—into a coherent global representation. 
Complementarily, Global Workspace Theory~\citep{baars1993cognitive,dehaene2001towards}
posits that specialized processors operate in parallel,
with information gaining access to a shared workspace being \emph{broadcast} globally,
achieving wide availability and exerting top-down influence on subsequent processing.
This broadcast mechanism finds support in hierarchical cortical processing~\citep{felleman1991distributed},
where top-down signals modulate lower-level representations.
Taken together, these perspectives motivate a hierarchical inductive bias:
\emph{local construction} of segment-level representations,
\emph{global integration} into a shared context,
and \emph{top-down broadcast} to condition subsequent attention.

Guided by this principle, we propose \textbf{HiCI}
(\textbf{Hi}erarchical \textbf{C}onstruction--\textbf{I}ntegration),
a hierarchical attention module that instantiates
\emph{construction, integration, and broadcast} within Transformer attention.
HiCI structures attention computation through three stages.
\textbf{Local construction} extracts segment-level representations
via cross-attention with a shared set of learnable query slots.
\textbf{Global integration} aggregates these local representations into a
compact shared context through multi-view statistical pooling and
attention-based weighting.
\textbf{Top-down broadcast} prepends both global and local representations 
to each segment's key--value sequence, conditioning token-level       
attention on hierarchical context while preserving segment-parallel computation.

We apply HiCI to pretrained LLaMA-2 models~\citep{touvron2023llama2} and Qwen3-8B~\citep{yang2025qwen3}, combining position interpolation~\citep{chen2023extending}
for context extension with FlashAttention-2~\citep{dao2024flashattention} for efficient long-sequence computation. Following LongLoRA’s parameter-efficient recipe~\citep{chen2024longlora}, we freeze the backbone and train only LoRA adapters, embedding and normalization layers, together with the proposed HiCI module. Despite adding only $\sim$4--5\% parameters during training, this enables context extension to 100K/64K tokens for LLaMA-2-7B/13B and to 48K tokens for Qwen3-8B. At inference, HiCI is optional: it can be applied during prefill to reduce time-to-first-token latency, or omitted in favour of standard full attention.

Extensive experiments on language modeling (PG-19~\citep{Rae2020Compressive},
Proof-pile~\citep{azerbayev2022proofpile}), retrieval (passkey and topic),
and instruction-following (LongBench~\citep{bai2024longbench}) benchmarks
demonstrate that HiCI consistently improves performance over strong baselines
across a wide range of tasks and context lengths. HiCI achieves 100\% passkey accuracy~\citep{mohtashami2023landmark} within the 32K training regime and
maintains substantially higher accuracy under direct extrapolation,
attains the best topic-retrieval~\citep{li2023how} accuracy among the evaluated open-source models,
and achieves higher accuracy than GPT-3.5-Turbo-16K~\citep{achiam2023gpt}
on the Code category of LongBench (+9.7\%). Ablation studies further identify global integration as the dominant contributor within HiCI, and show that explicit hierarchical conditioning is essential beyond grouped attention alone.

In summary, our contributions are:                                                                             
\begin{itemize}                                                                                                
\item We propose HiCI, a hierarchical attention module that instantiates                                       
construction--integration--broadcast as an explicit inductive bias for                                         
long-context modeling in Transformers.                                                                 
\item We show that HiCI generalises across model families and scales, requiring only $\sim$4--5\% additional parameters to extend LLaMA-2-7B/13B to 100K/64K tokens and Qwen3-8B to 48K, while consistently improving perplexity and downstream performance on retrieval and instruction-following tasks.

\item Ablation studies identify global integration as the dominant component, and reveal an optimal compact slot budget. Attention visualizations further show that deeper layers increasingly attend to global representations, indicating emergent hierarchical information routing.

  \end{itemize}

\section{Related Work}
\subsection{Efficient Attention Mechanisms}

  The quadratic complexity of self-attention has motivated extensive research on efficient alternatives.
  \textbf{Sparse attention} restricts the attention pattern to reduce computation:
  Longformer~\citep{beltagy2020longformer} employs sliding windows augmented with task-specific global tokens,
  BigBird~\citep{zaheer2020bigbird} combines local, global, and random attention to achieve linear complexity with theoretical guarantees,
  and LongNet~\citep{DBLP:journals/corr/abs-2307-02486} uses dilated attention with exponentially increasing receptive fields across heads.
  \textbf{Linear attention}~\citep{katharopoulos2020transformers} approximates softmax via kernel decomposition, enabling $O(n)$ complexity.
  However, kernel-based approximations exhibit degraded performance on retrieval-intensive tasks~\citep{arora2024zoology},
  and predefined sparsity patterns limit adaptability to diverse long-range dependencies.

\subsection{Context Window Extension for LLMs}
LLMs are typically pre-trained with fixed context lengths (e.g., 4{,}096 for Llama-2),
and context extension has been pursued through \emph{positional scaling} and \emph{efficient long-context adaptation}.
\textbf{Positional encoding} methods modify RoPE-style position representations to
improve length extrapolation. Position Interpolation (PI)~\citep{chen2023extending}
rescales position indices and relies on substantial continued training to adapt to
longer contexts. Subsequent schemes such as YaRN~\citep{peng2024yarn} and
LongRoPE~\citep{ding2024longrope} introduce frequency-aware or non-uniform scaling,
reducing the amount of long-context continued training relative to PI.
These methods address \emph{where} to attend but retain quadratic complexity and do not alter how attention organizes context.
\textbf{Training and adaptation} methods address long-context fine-tuning with varying efficiency. Early work such as Focused Transformer~\citep{tworkowski2023focused} employs specialized
training objectives, but remains computationally intensive (128 TPUs).
More efficient alternatives have since emerged:
LongLoRA~\citep{chen2024longlora} combines shifted sparse attention with LoRA,
enabling 100k context on 8$\times$A100;
PoSE~\citep{zhu2024pose} simulates long positions within fixed windows;
LongAlign~\citep{bai2024longalign} accelerates training via packing strategies. Despite substantially reducing adaptation cost, these methods lack an explicit
mechanism for organizing and globally sharing contextual information. HiCI builds upon LongLoRA while introducing hierarchical context organization,
constructing local-to-global abstractions that condition token-level attention
(Section~\ref{sec:method}).

\subsection{Segment-based Long-context Modeling}
The $\mathcal{O}(L^2)$ cost of self-attention motivates segment-wise processing, trading direct cross-segment interaction for efficiency. Existing approaches differ in how they restore this connectivity. \textbf{Recurrence-based methods} propagate information through sequential state updates across segments.  
Transformer-XL~\citep{dai2019transformerxl} caches hidden states from prior segments and attends to them as
extended context, RMT~\citep{bulatov2022recurrent} transmits learnable memory tokens across segment       
boundaries, and Block-Recurrent Transformer~\citep{hutchins2022block} combines block-level recurrence with 
attention for improved parallelism. Despite their effectiveness, sequential dependencies limit parallel    
training and risk information attenuation over long distances. \textbf{Compression-based methods} summarize past segments into fixed-capacity
representations. Compressive Transformer~\citep{Rae2020Compressive} learns to compress older memories, while Infini-attention~\citep{munkhdalai2024leave} incrementally updates     
a compressive state via linear attention.     
These approaches bound memory but sacrifice fine-grained fidelity~\citep{xu2026rethinking, xu2026vla}.
\textbf{Hierarchical methods} construct multi-level abstractions.       
HMT~\citep{he2025hmt} maintains a memory hierarchy with segment summarization, Block Transformer~\citep{ho2024block} separates global block-level and local token-level attention into distinct modules, bypassing token-level KV cache for faster inference, and EM-LLM~\citep{fountas2025emllm} segments via Bayesian surprise inspired by episodic memory. 

In addition to these explicit mechanisms, LongLoRA~\citep{chen2024longlora} partitions attention into local groups and enables implicit interaction via shifted grouping across heads.
In summary, existing segment-based methods restore cross-segment connectivity at the cost of parallelism,         
fidelity, or explicit semantic organization. 
Motivated by Construction–Integration~\citep{kintsch1988role} and    
Global Workspace Theory~\citep{baars1993cognitive}, 
HiCI addresses these limitations: segment-local  
representations are constructed via cross-attention, integrated into global context, and both are concatenated    
with original tokens in KV space—enabling parallel,
semantically explicit conditioning over long contexts.

\subsection{Global Workspace Architectures and Latent Representations}

Global Workspace Theory (GWT)~\citep{baars1993cognitive} posits that
cognition arises from a shared broadcast mechanism that integrates and
disseminates information across specialised modules.
\citet{vanrullen2021deep} translate this framework into design principles
for deep networks; \citet{DBLP_conf_iclr_GoyalDLBKRBBMB22, zeng2024a} instantiate a shared workspace for inter-module communication in modular networks, while \citet{hong2024concept} extend the motif to concept-centric transformer representations.
Complementary to architectural coordination, recent work explores whether discrete tokens are the appropriate medium for language model computation. Coconut~\citep{hao2025training} replaces chain-of-thought
tokens with continuous latent states for multi-step reasoning, whereas Large Concept Models~\citep{barrault2024large} predict sentence-level embeddings instead of next tokens.
HiCI is inspired by the GWT intuition of constructing and integrating global context, but differs in architectural scope: rather than coordinating distinct functional modules or concept abstractions, it performs cross-segment context aggregation within a single transformer layer for long-context language modeling. Unlike latent-space approaches,
HiCI preserves standard next-token autoregressive prediction; its global context vectors are auxiliary attention prefixes injected as keys and values rather than latent reasoning states or prediction targets.

\section{Hierarchical Construction–Integration Attention}
\label{sec:method}
We present HiCI, a lightweight attention module that instantiates a
cognitively motivated inductive bias for long-context modeling.
HiCI organizes attention computation into three stages—local construction,
global integration, and top-down broadcast—mirroring the hierarchical process of human discourse comprehension.

\subsection{Overview}
  \label{sec:overview}
Standard self-attention induces pairwise interactions among all $T$ tokens, resulting in $\mathcal{O}(T^2)$ computational complexity~\citep{vaswani2017attention}. A widely adopted alternative is segmented attention, which partitions the input into fixed-length segments and restricts attention to within-segment interactions, reducing the complexity to $\mathcal{O}(T \cdot S)$. However, such formulations lack an explicit mechanism for propagating information across segments. HiCI addresses this limitation through structured context conditioning: it dynamically constructs compact local and global representations from the input and injects them back into each block's attention computation.

Given an input sequence $X \in \mathbb{R}^{T \times d}$, assuming $T$ is divisible
by the segment length $S$, we partition it into
$N = T / S$ segments ${X_1, \ldots, X_N}$ and proceed as follows (Figure~\ref{fig:method}):

\begin{enumerate}[leftmargin=*, itemsep=2pt, topsep=3pt]
\item \textbf{Local Construction} (\S\ref{sec:local}): For each
segment $X_i \in \mathbb{R}^{S \times d}$, cross-attention with $M$ learnable query slots
extracts a local representation $L_i \in \mathbb{R}^{M \times d}$.
\item \textbf{Global Integration} (\S\ref{sec:global}): The local representations $\{L_i\}_{i=1}^{N}$ are aggregated into a shared global context $G \in \mathbb{R}^{K \times d}$ via multi-view statistical pooling followed by attention-based weighting.

\item \textbf{Top-down Broadcast} (\S\ref{sec:broadcast}):
The global context $G$ and segment-specific abstraction $L_i$ are prepended to the key--value sequence of each segment $X_i$, conditioning token-level updates on hierarchical context while preserving parallelism across segments.
\end{enumerate}
Throughout, the cardinalities $M$ and $K$ are fixed constants independent of the sequence length $T$.

\begin{figure*}[t]
\centering
\includegraphics[width=0.75\textwidth]{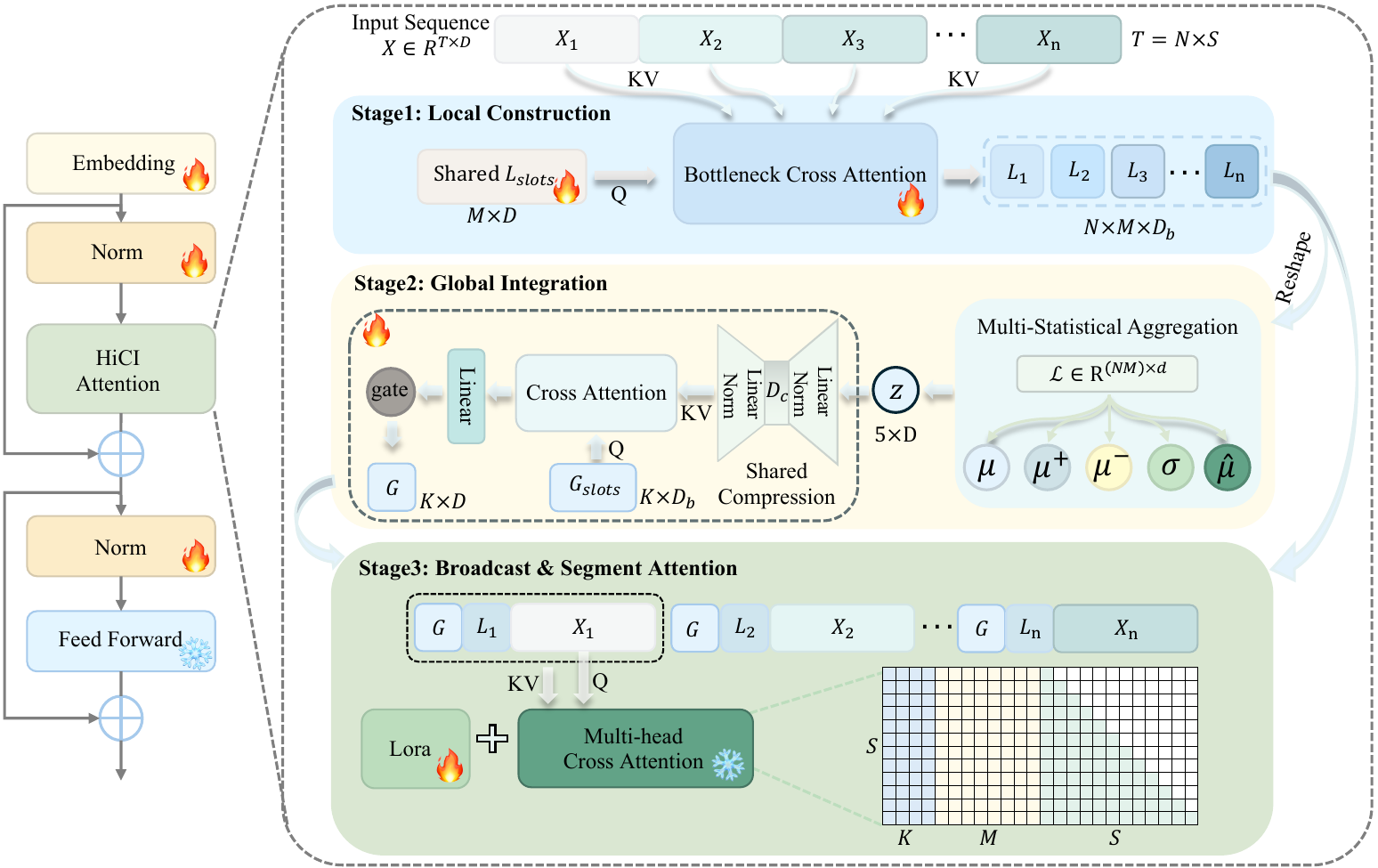}
\caption{
\textbf{Overview of HiCI.} \textit{Left:} HiCI integrated into a Transformer block; trainable components are highlighted.
\textit{Right:} HiCI constructs hierarchical context through three stages.     
\textbf{(1) Local Construction:} the input sequence is partitioned into $N$     
segments, and cross-attention with $M$ learnable query slots extracts a local  
representation $L_i$ from each segment. 
\textbf{(2) Global Integration:} local representations $\{L_i\}_{i=1}^N$ are aggregated into a shared global context $G$ via multi-view statistical pooling and attention-based weighting.
\textbf{(3) Top-down Broadcast:} $G$ and $L_i$ are prepended to each segment's key--value sequence, conditioning attention on hierarchical context while preserving parallelism across segments.
At inference, HiCI is optionally applied during prefill, while autoregressive decoding uses standard attention.
}
\label{fig:method}
\end{figure*}

\subsection{Local Construction}
\label{sec:local}

The first stage performs \emph{local construction}, distilling each input
segment $X_i \in \mathbb{R}^{S \times d}$ into a compact representation
$L_i \in \mathbb{R}^{M \times d}$, where $M \ll S$ is a small, sequence-length-independent constant, 

consistent with the limited capacity of human working memory~\citep{miller1956magical,cowan2001magical}.

\textbf{Bottleneck Cross-Attention.}
We introduce $M$ learnable slot vectors $L_{\text{slot}} \in \mathbb{R}^{M \times d}$,
shared across all segments, which serve as queries attending to segment tokens
via multi-head cross-attention.

To improve parameter efficiency and induce abstraction, attention is computed in a low-dimensional subspace $\mathbb{R}^{d_b}$ with $d_b \ll d$.

Formally, for each segment $X_i \in \mathbb{R}^{S \times d}$, the local
representation $L_i \in \mathbb{R}^{M \times d}$ is computed as
\begin{align}
\tilde{L}_i
&=
\mathrm{softmax}\!\left(
\frac{(L_{\text{slot}} W_Q^{\ell})
(X_i W_K^{\ell})^\top}{\sqrt{d_k}}
\right)
(X_i W_V^{\ell}), \\
L_i
&=
\tilde{L}_i W_O^{\ell},
\end{align}
where
$\{W_Q^{\ell}, W_K^{\ell}, W_V^{\ell}\} \in \mathbb{R}^{d \times d_b}$
and
$W_O^{\ell} \in \mathbb{R}^{d_b \times d}$
are learned projections, with $H$ attention heads of dimension $d_k = d_b / H$.

The bottleneck $(M,d_b)$ defines a fixed-capacity interface that favors salient
segment-level structure over fine-grained token detail. Aggregating the resulting $\{L_i\}_{i=1}^{N}$ yields
$L \in \mathbb{R}^{N \times M \times d}$ for subsequent integration. A formal treatment of this constraint is given in Appendix~\ref{app:theory}.

\subsection{Global Integration}
\label{sec:global}
Given the stacked local representations $L \in \mathbb{R}^{N \times M \times d}$, the global integration stage consolidates segment-level information into a compact global context $G \in \mathbb{R}^{K \times d}$, where a small $K$ reflects the
capacity constraints of a global workspace~\citep{baars1993cognitive}.

\textbf{Multi-View Statistical Aggregation.}
We collapse the segment and slot dimensions of $L \in \mathbb{R}^{N \times M \times d}$ into a single axis, yielding $\mathcal{L} \in \mathbb{R}^{(N M) \times d}$, and compute five
complementary statistics over this axis:

\begin{align}
\boldsymbol{\mu} &= \frac{1}{N M} \sum_{j=1}^{N M} \mathcal{L}_j, \label{eq:stat-mean} \\
\boldsymbol{\mu}^{+} &= \max_{j} \mathcal{L}_j, \quad
\boldsymbol{\mu}^{-} = \min_{j} \mathcal{L}_j, \label{eq:stat-minmax} \\
\boldsymbol{\sigma} &= \sqrt{\frac{1}{N M} \sum_{j=1}^{N M} (\mathcal{L}_j - \boldsymbol{\mu})^2}, \label{eq:stat-std} \\
\hat{\boldsymbol{\mu}} &= \boldsymbol{\mu} / \|\boldsymbol{\mu}\|_2. \label{eq:stat-dir}
\end{align}

Each statistic lies in $\mathbb{R}^d$ and captures a complementary aspect of the
aggregated representations: $\boldsymbol{\mu}$ reflects central tendency,
$\boldsymbol{\mu}^{+}$ and $\boldsymbol{\mu}^{-}$ capture element-wise extremal
activations, $\boldsymbol{\sigma}$ measures dispersion, and
$\hat{\boldsymbol{\mu}}$ encodes directional information independent of
magnitude via $\ell_2$-normalization.

\textbf{Shared Compression.}
We organize the five statistics into a matrix
\begin{equation}
\mathbf{Z} = \big[\, \boldsymbol{\mu};\ \boldsymbol{\mu}^{+};\ \boldsymbol{\mu}^{-};\ \boldsymbol{\sigma};\ \hat{\boldsymbol{\mu}} \,\big] \in \mathbb{R}^{5 \times d},
\label{eq:stats-matrix}
\end{equation}
where each row corresponds to one statistical view.
Rather than learning separate projections, we apply a shared two-stage compression
$\phi: \mathbb{R}^{d} \rightarrow \mathbb{R}^{d_b}$:
\begin{equation}
\tilde{\mathbf{Z}} = \phi(\mathbf{Z}) = \psi_b \circ \psi_c(\mathbf{Z}),
\label{eq:shared-compress}
\end{equation}
where $\psi_c(\cdot) = \mathrm{LayerNorm}(\cdot \, W_c)$ with $W_c \in \mathbb{R}^{d \times d_s}$,
and $\psi_b(\cdot) = \mathrm{LayerNorm}(\cdot \, W_b)$ with $W_b \in \mathbb{R}^{d_s \times d_b}$.
The intermediate bottleneck $d_s < d_b \ll d$ induces abstraction via an information
bottleneck~\citep{tishby2000information}, while parameter sharing enforces
consistent compression across heterogeneous statistical views.

\textbf{Attention-Based Selection.}
We introduce $K$ learnable query vectors $Q_G \in \mathbb{R}^{K \times d_b}$
that attend to the compressed statistics via multi-head cross-attention.
Formally,
\begin{align}
G_c &= \mathrm{softmax}\!\left(
\frac{(Q_G W_Q^{g})(\tilde{\mathbf{Z}} W_K^{g})^\top}{\sqrt{d_b/H}}
\right)
(\tilde{\mathbf{Z}} W_V^{g}),
\label{eq:global-attn}
\end{align}
where $\{W_Q^{g}, W_K^{g}, W_V^{g}\} \in \mathbb{R}^{d_b \times d_b}$ are learned projections
with $H$ attention heads and $d_b$ as in \S\ref{sec:local}.
The output is then projected back to the model dimension with a learnable gate:
\begin{equation}
G = G_c W_{\text{exp}} \cdot \alpha, \quad \alpha = \ln(1 + e^{\beta}),
\label{eq:expand}
\end{equation}
where $W_{\text{exp}} \in \mathbb{R}^{d_b \times d}$ and $\beta \in \mathbb{R}$ is a learnable scalar.
The constraint $\alpha > 0$ ensures stable scaling of the global context.
The resulting $G \in \mathbb{R}^{K \times d}$ serves as the global context for top-down broadcast (\S\ref{sec:broadcast}).

\subsection{Top-down Broadcast}
\label{sec:broadcast}

The final stage performs top-down broadcast, conditioning segment-level
attention on both the globally integrated context $G$ and the corresponding
local abstraction $L_i$.

For each segment $X_i \in \mathbb{R}^{S \times d}$, we form a context-augmented sequence by concatenating the global and local representations with the segment tokens:
\[
[\, G;\, L_i;\, X_i \,] \in \mathbb{R}^{(K+M+S)\times d}.
\]
The augmented sequence is projected into the key--value space as
\begin{equation}
\tilde{K}_i = [\, G;\, L_i;\, X_i \,] W_K^{b}, \qquad
\tilde{V}_i = [\, G;\, L_i;\, X_i \,] W_V^{b},
\label{eq:kv_aug}
\end{equation}
where $W_K^{b}, W_V^{b} \in \mathbb{R}^{d \times d}$. 

Queries are derived exclusively from segment tokens as
$Q_i = X_i W_Q^{b}$, where $W_Q^{b} \in \mathbb{R}^{d \times d}$.

Attention over the augmented context yields a context-conditioned update:
\begin{equation}
\tilde{X}_i
=
\mathrm{softmax}\!\left(
\frac{Q_i \tilde{K}_i^\top}{\sqrt{d/H}}
\right)
\tilde{V}_i
\in \mathbb{R}^{S \times d},
\label{eq:broadcast-attn}
\end{equation}
where $H$ is the number of attention heads.

Since each segment attends to its augmented context independently,
all $N$ segments can be processed in parallel. The refined segments are concatenated to form the output:
\begin{equation}
\tilde{X}
=
\mathrm{Concat}(\tilde{X}_1, \dots, \tilde{X}_N)
\in \mathbb{R}^{T \times d}.
\label{eq:hici-output}
\end{equation}

By jointly attending over all $K{+}M{+}S$ positions under a unified softmax,
each token integrates global, local, and segment-level context,
implementing top-down modulation (see Appendix~\ref{app:theory} for analysis).

\section{Experiments}
\label{sec:experiments}
In this section, we evaluate the effectiveness of HiCI across language modeling(\S\ref{sec:lm}), retrieval (\S\ref{sec:retrieval}), and downstream benchmarks      
(\S\ref{sec:downstream}), followed by ablation studies (\S\ref{sec:ablation}). Additional attention   
analysis is given in the Appendix~\ref{app:attention_analysis}.

\subsection{Experimental Setup}
\label{sec:exp_setup}
\begin{table*}[ht!]
\centering  
\caption{Perplexity ($\downarrow$) on PG-19 and Proof-pile test sets for LLaMA-2-7B/13B and Qwen3-8B continually pre-trained on RedPajama for 1k training steps. LLaMA-2 models are trained with context lengths ranging from 8K to 100K and evaluated up to 100K; Qwen3-8B extends its original 32K context window to 48K and is evaluated up to 48K, with an additional 500-step variant reported.
}
\label{tab:ppl}         
\small 
\setlength{\tabcolsep}{1.9pt}
\renewcommand{\arraystretch}{1}
\newcommand{\shade}{\cellcolor{cyan!8}}   
\begin{tabular}{ll l|ccccccc|ccccccc@{\hspace{\tabcolsep}}}
\toprule
& & &
\multicolumn{7}{c|}{\textbf{PG-19}} &
\multicolumn{7}{c}{\textbf{Proof-pile}} \\
\cmidrule(lr){4-10} \cmidrule(lr){11-17}
\multirow{-2}{*}{\textbf{Base Model}}
& \multirow{-2}{*}{\textbf{Train}}
& \multirow{-2}{*}{\textbf{Method}}
& \textbf{2K} & \textbf{4K} & \textbf{8K} & \textbf{16K} & \textbf{32K} & \textbf{64K} & \textbf{100K}
& \textbf{2K} & \textbf{4K} & \textbf{8K} & \textbf{16K} & \textbf{32K} & \textbf{64K} & \textbf{100K} \\
\midrule
\multirow{10}{*}[-2.2ex]{LLaMA-2-7B}
& \multirow{2}{*}{4k}
& LLaMA-2-7B
& 7.49 & 7.15 & $>\!10^2$ & $>\!10^2$ & $>\!10^2$ & $>\!10^2$ & $>\!10^2$
& 3.21 & 2.91 & $>\!10^2$ & $>\!10^2$ & $>\!10^2$ & $>\!10^2$ & $>\!10^2$ \\
& & ChunkLLaMA
& 7.49 &  7.15 &  6.98 &  6.96 &  7.08 &  15.15 &  --
&  3.21 &  2.91 &  2.75 &  2.69 & 2.70 &  2.75 &  -- \\
\cmidrule(l){2-17}
& \multirow{2}{*}{8K}
& LongLoRA
& 7.70 & 7.35 & 7.14 & -- & -- & -- & --
& 3.20 & 2.91 & 2.72 & -- & -- & -- & -- \\
& & HiCI
& \shade 7.27 & \shade 7.01 & \shade 6.93 & \shade -- & \shade -- & \shade -- & \shade --
& \shade 3.07 & \shade 2.82 & \shade 2.65 & \shade -- & \shade -- & \shade -- & \shade -- \\
\cmidrule(l){2-17}
& \multirow{2}{*}{16K}
& LongLoRA
& 7.65 & 7.28 & 7.02 & 6.86 & -- & -- & --
& 3.17 & 2.87 & 2.66 & 2.51 & -- & -- & -- \\
& & HiCI
& \shade 7.53 & \shade 7.21 & \shade 6.96 & \shade 6.84 & \shade -- & \shade -- & \shade --
& \shade 3.15 & \shade 2.84 & \shade 2.61 & \shade 2.47 & \shade -- & \shade -- & \shade -- \\
\cmidrule(l){2-17}
& \multirow{2}{*}{32K}
& LongLoRA
& 8.29 & 7.83 & 7.54 & 7.35 & 7.22 & -- & --
& 3.35 & 3.01 & 2.78 & 2.61 & 2.50 & -- & -- \\
& & HiCI
& \shade 7.87 & \shade 7.50 & \shade 7.26 & \shade 7.09 & \shade 7.11 & \shade -- & \shade --
& \shade 3.21 & \shade 2.87 & \shade 2.71 & \shade 2.58 & \shade 2.49 & \shade -- & \shade -- \\
\cmidrule(l){2-17}
& \multirow{2}{*}{100K}
& LongLoRA
& 8.38 & 7.90 & 7.57 & 7.33 & 7.16 & 7.06 & 7.04
& 3.36 & 3.01 & 2.78 & 2.60 & 2.58 & 2.57 & 2.52 \\
& & HiCI
& \shade 7.81 & \shade 7.72 & \shade 7.45 & \shade 7.26 & \shade 7.08 & \shade 6.97 & \shade 6.95
& \shade 3.27 & \shade 2.86 & \shade 2.73 & \shade 2.54 & \shade 2.48 & \shade 2.46 & \shade 2.43 \\
\midrule
\multirow{10}{*}[-2.2ex]{LLaMA-2-13B}
& \multirow{2}{*}{4k}
& LLaMA-2-13B
& 6.86 & 6.55 & $>\!10^2$ & $>\!10^2$ & $>\!10^2$ & $>\!10^2$ & $>\!10^2$
& 3.04 & 2.78 & 76.05 & $>\!10^2$ & $>\!10^2$ & $>\!10^2$ & $>\!10^2$ \\
& & ChunkLLaMA
&  6.86 &  6.55 &  6.37 &  6.35 &  6.46 &  14.36 &  --
&  3.04 &  2.78 &  2.62 &  2.55 &  2.58 &  2.64 & -- \\
\cmidrule(l){2-17}
& \multirow{2}{*}{8K}
& LongLoRA
& 7.03 & 6.73 & 6.58 & -- & -- & -- & --
& 3.04 & 2.77 & 2.60 & -- & -- & -- & -- \\
& & HiCI
& \shade 6.68 & \shade 6.46 & \shade 6.34 & \shade -- & \shade -- & \shade -- & \shade --
& \shade 2.91 & \shade 2.69 & \shade 2.52 & \shade -- & \shade -- & \shade -- & \shade -- \\
\cmidrule(l){2-17}
& \multirow{2}{*}{16K}
& LongLoRA
& 7.05 & 6.70 & 6.47 & 6.31 & -- & -- & --
& 3.03 & 2.74 & 2.55 & 2.41 & -- & -- & -- \\
& & HiCI
& \shade 6.95 & \shade 6.65 & \shade 6.43 & \shade 6.28 & \shade -- & \shade -- & \shade --
& \shade 2.99 & \shade 2.73 & \shade 2.53 & \shade 2.40 & \shade -- & \shade -- & \shade -- \\
\cmidrule(l){2-17}
& \multirow{2}{*}{32K}
& LongLoRA
& 7.05 & 6.70 & 6.47 & 6.31 & 6.20 & -- & --
& 3.03 & 2.74 & 2.55 & 2.41 & 2.32 & -- & -- \\
& & HiCI
& \shade 6.94 & \shade 6.56 & \shade 6.39 & \shade 6.25 & \shade 6.17 & \shade -- & \shade --
& \shade 2.94 & \shade 2.68 & \shade 2.40 & \shade 2.35 & \shade 2.26 & \shade -- & \shade -- \\
\cmidrule(l){2-17}
& \multirow{2}{*}{64K}
& LongLoRA
& 7.63 & 7.21 & 6.94 & 6.75 & 6.62 & 6.53 & --
& 3.05 & 2.76 & 2.57 & 2.42 & 2.32 & 2.25 & -- \\
& & HiCI
& \shade 7.40 & \shade 7.06 & \shade 6.81 & \shade 6.62 & \shade 6.47 & \shade 6.39 & \shade --
& \shade 2.96 & \shade 2.63 & \shade 2.38 & \shade 2.31 & \shade 2.20 & \shade 2.17 & \shade -- \\
\midrule
\multicolumn{3}{c|}{} &


\textbf{2K} & \textbf{4K} & \textbf{8K} & \textbf{16K} & \textbf{32K} & \textbf{48K} & {--} &
\textbf{2K} & \textbf{4K} & \textbf{8K} & \textbf{16K} & \textbf{32K} & \textbf{48K} & {--} \\

\midrule

\multirow{3}{*}[-0.5ex]{Qwen3-8B}
& {32k}
& Qwen3-8B
& 13.26 & 12.58 & 12.09 & 11.72 & 12.76 & 11.32 & --
& 3.24 & 2.94 & 2.73 & 2.59 & 2.49 & 2.46 & -- \\
\cmidrule(l){2-17}

& \multirow{2}{*}{48K}
& HiCI (500s)
& \shade 11.71 & \shade 11.06 & \shade 10.59 & \shade 10.24 & \shade 9.98 & \shade 9.89 & \shade --
& \shade 3.01 & \shade 2.73 & \shade 2.54 & \shade 2.41 & \shade 2.32 & \shade 2.29 & \shade -- \\

& 
& HiCI
& \shade 11.46 & \shade 10.84 & \shade 10.38 & \shade 10.06 & \shade 9.82 & \shade 9.73 & \shade --
& \shade 2.95 & \shade 2.68 & \shade 2.50 & \shade 2.37 & \shade 2.30 & \shade 2.26 & \shade -- \\
\bottomrule
\end{tabular}
\end{table*}

\textbf{Models.}
We evaluate HiCI on pretrained LLaMA-2~\citep{touvron2023llama2} 7B/13B and Qwen3-8B~\citep{yang2025qwen3} models, extending their context windows using Position Interpolation~\citep{chen2023extending} to 100K/64K tokens for LLaMA-2-7B/13B and 48K for Qwen3-8B.

\textbf{Training.} 
Following LongLoRA~\citep{chen2024longlora}, we perform two-stage LoRA
fine-tuning: continued pretraining on RedPajama~\citep{together2023redpajama} with the next-token prediction objective, then instruction tuning on LongAlpaca-12k~\citep{chen2024longlora}, training only the HiCI module,
LoRA adapters, embeddings, and normalization layers.
Optimization is performed with AdamW ($\beta_1{=}0.9$, $\beta_2{=}0.95$, weight
decay $0$), using a learning rate of $2{\times}10^{-5}$ for the backbone and
$2{\times}10^{-4}$ for HiCI with a 20-step linear warmup.
Unless otherwise specified, we train for 1{,}000 steps with per-device
batch size 1 and gradient accumulation 8, yielding an effective batch size of 64. All experiments use bf16 precision with DeepSpeed ZeRO-2~\citep{rasley2020deepspeed} and Flash-Attention2~\citep{dao2024flashattention}, running on 8$\times$H100 GPUs for LLaMA-2 and 8$\times$H200 GPUs for Qwen3-8B; full hyperparameter details are provided in Appendix~\ref{app:hyperparameters}.

\textbf{Evaluation.}
We adopt the two-stage evaluation protocol of LongLoRA.
Stage~1 assesses long-context language modeling and retrieval: we report perplexity on   
PG-19~\citep{Rae2020Compressive} and Proof-pile~\citep{azerbayev2022proofpile} using a sliding window with stride 
256~\citep{DBLP:journals/corr/abs-2108-12409} 
and the same hierarchical attention as training, along with passkey 
retrieval~\citep{mohtashami2023landmark} and topic retrieval~\citep{li2023how}.
Stage~2 evaluates downstream instruction-following on LongBench~\citep{bai2024longbench}  
under two inference      
modes: standard full attention and HiCI attention during prefill. 

\subsection{Language Modeling}
\label{sec:lm}
We evaluate perplexity on PG-19~\citep{Rae2020Compressive} and Proof-pile~\citep{azerbayev2022proofpile}, comparing HiCI against LongLoRA~\citep{chen2024longlora} and ChunkLLaMA~\citep{DBLP_conf_icml_An0ZGQZK24} for LLaMA-2-7B/13B (training lengths 8K--100K, evaluation up to 100K), and against the base model for Qwen3-8B under context extension from 32K to 48K, as reported in Table~\ref{tab:ppl}. The longest-context settings (100K for LLaMA-2-7B and 64K for LLaMA-2-13B) use DeepSpeed Stage-3~\citep{rajbhandari2020zero} with adjusted group configurations; details are provided in Appendix~\ref{app:hyperparameters}.
HiCI reduces perplexity across all evaluation lengths for both model families, indicating that hierarchical context organization enhances language modeling quality beyond context extension alone.
For LLaMA-2-7B/13B, HiCI consistently outperforms LongLoRA across model scales and training lengths, with ChunkLLaMA serving as a training-free reference. 
The performance gap is more pronounced at shorter evaluation lengths: for the 100K-trained LLaMA-2-7B, HiCI achieves a 6.8\% reduction at 2K evaluation, compared to 1.3\% at 100K, suggesting better preservation of short-range modeling quality. For Qwen3-8B, HiCI achieves a 14\% perplexity reduction on PG-19 and 8\% on Proof-pile at 48K, demonstrating that the hierarchical context mechanism transfers to models with stronger native context; where a 500-step variant already captures most of the gain.

\subsection{Retrieval-based Evaluation}
\label{sec:retrieval}

\textbf{Topic Retrieval.} 
We evaluate on the LongChat topic retrieval task~\citep{li2023how}, which requires identifying a target topic  
from multi-turn dialogues spanning 3K--16K tokens. As shown in Table~\ref{tab:topic}, while closed-source models such as GPT-4o-mini-128K~\citep{hurst2024gpt} and Claude-3.5-Sonnet-200K~\citep{anthropic2024claude} achieve perfect accuracy, open-source alternatives show notable degradation: models with shorter context windows (e.g., ChatGLM2-6B-8k~\citep{du2022glm}, MPT-30B-Chat-8k~\citep{MosaicML2023mpt}, and Llama-3-8B-Instruct-8K~\citep{grattafiori2024llama}) fail beyond their training length, and even MPT-7B-StoryWriter-65K~\citep{MosaicML2023mpt} achieves only 0.28--0.46 across all lengths. In contrast,      
HiCI-13B-16K achieves the best accuracy among open-source models, matching 100\% up to 13K and reaching 0.94 at 16K, compared to 0.90 for LongChat-13B-16K~\citep{li2023how} and 0.86 for LongLoRA-13B-16K~\citep{chen2024longlora}. We conjecture that HiCI's stability is driven by a hierarchical inductive bias: segment-level construction learns content-dependent representations, while global integration forms position-invariant contextual representations, reducing sensitivity to where evidence appears in the sequence.

\begin{table}[h!]
\centering
\caption{Topic retrieval accuracy on LongChat~\citep{li2023how}. We compare   
HiCI against both proprietary models and open-source long-context LLMs across 3K--16K context lengths. HiCI-13B-16K 
matches proprietary model performance up to 13K and outperforms all             
open-source baselines at 16K.}  
\label{tab:topic}
\small
\setlength{\tabcolsep}{5pt}
\renewcommand{\arraystretch}{1}
\begin{tabular}{lccccc}
\toprule
\textbf{Model}& \textbf{3K} & \textbf{6K} & \textbf{10K} & \textbf{13K} & \textbf{16K} \\
\midrule
\rowcolor{gray!6}
GPT-4o-mini-128K           & 1.00 & 1.00 & 1.00 & 1.00 & 1.00 \\
Claude-3.5-Sonnet-200K             & 1.00 & 1.00 & 1.00 & 1.00 & 1.00 \\
\midrule
\rowcolor{gray!6}
MPT-30B-Chat-8K             & 0.96 & 1.00 & 0.76 & --   & --   \\
ChatGLM2-6B-8K              & 0.88 & 0.46 & 0.02 & 0.02 & 0.02 \\
\rowcolor{gray!6}
MPT-7B-StoryWriter-65K      & 0.46 & 0.46 & 0.28 & 0.34 & 0.36 \\
LongChat-13B-16K            & 1.00 & 1.00 & 1.00 & 0.98 & 0.90 \\
\rowcolor{gray!6}
LongLoRA-13B-16K$^\dagger$  & 1.00 & 0.96 & 1.00 & 0.98 & 0.86 \\
Llama-3-8B-Instruct-8K  & 1.00 & 1.00 & 0.00 & 0.00 & 0.00 \\
\rowcolor{softpink}
\textbf{HiCI-13B-16K (Ours)} &
\textbf{1.00} & \textbf{1.00} & \textbf{1.00} & \textbf{1.00} & \textbf{0.94} \\
\bottomrule
\end{tabular}
\raggedright
{\footnotesize $^\dagger$ Evaluated with official LoRA weights.}
\end{table}

\begin{figure*}[t]
\centering
\includegraphics[width=\textwidth]{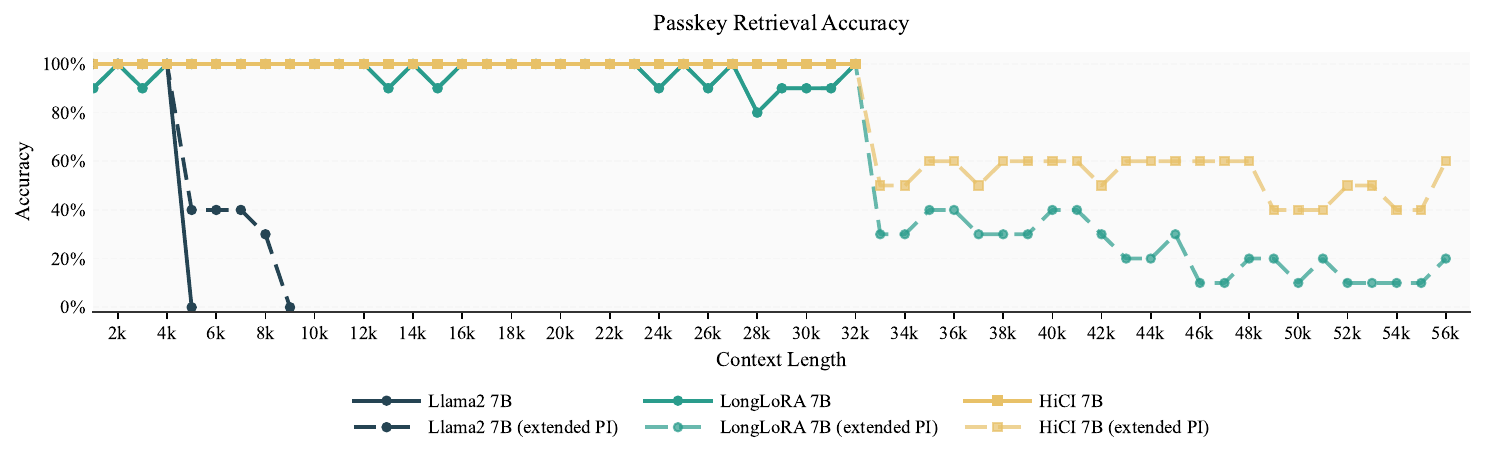}
\caption{Passkey retrieval accuracy for LongLoRA-7B, HiCI-7B (both fine-tuned at 32K), and base LLaMA-2-7B. HiCI achieves 100\% accuracy within the training length and extrapolates more gracefully to 56K via position interpolation without additional fine-tuning.} 
\label{fig:passkey}
\end{figure*}

\textbf{Passkey Retrieval.}
We evaluate passkey retrieval following~\citet{mohtashami2023landmark}, where models are required to locate and output a random passkey embedded within long distractor text.
For each context length, we conduct 10 trials with randomized passkey values and insertion positions.
Figure~\ref{fig:passkey} compares HiCI-7B-32K, LongLoRA-7B-32K~\citep{chen2024longlora}, and the base LLaMA-2-7B model~\citep{touvron2023llama2}.
Within the 32K training regime, HiCI achieves 100\% retrieval accuracy across all evaluated lengths, whereas LongLoRA exhibits non-monotonic behavior with accuracy fluctuating between 80\% and 100\%, and the base LLaMA-2-7B model~\citep{touvron2023llama2}, constrained by its native 4K context window, fails to retrieve passkeys beyond this length.
To assess length extrapolation, we extend the maximum context at inference time to 56K using position interpolation (PI)~\citep{chen2023extending}, without any additional fine-tuning, following~\citet{chen2024longlora}. Beyond the 32K, both fine-tuned models exhibit degradation, consistent with the known sensitivity of RoPE-based positional encoding to out-of-distribution positions.
Notably, HiCI degrades more gracefully, maintaining 40--60\% retrieval accuracy over the 33K--56K range, compared to LongLoRA’s 10--30\% accuracy under the same setting. These results suggest that HiCI's training-time inductive bias may yield representations more robust to position extrapolation.

\subsection{Downstream Tasks}
\label{sec:downstream}

\begin{table*}[t]
\caption{Results (\%) on LongBench~\citep{bai2024longbench} benchmark. Best in \textbf{bold}, second \underline{underlined}. {\color{teal}$\uparrow$} indicates improvement over LongLoRA-7B-16K (direct baseline) and {\color{orange}$\uparrow$} highlights top-2 gains for each HiCI variant.}
\label{tab:longbench}  
\centering  
\small      
\setlength{\tabcolsep}{4.85pt}  
\begin{tabular}{l|cccccc|ccc}  
\toprule     
\multirow{2}{*}[-0.85ex]{\textbf{Model}}     
& \multirow{2}{*}[-0.85ex]{\textbf{Single-Doc QA}
}  
& \multirow{2}{*}[-0.85ex]{\textbf{Multi-Doc QA}} 
& \multirow{2}{*}[-0.85ex]{\textbf{Summ}} 
& \multirow{2}{*}[-0.85ex]{\textbf{Few-shot}}  
& \multirow{2}{*}[-0.85ex]{\textbf{Synthetic}} 
& \multirow{2}{*}[-0.85ex]{\textbf{Code}}  
& \multicolumn{3}{c}{\textbf{Overall}} \\ 
\cmidrule(lr){8-10}  
&  &  &  &  &  &  & \textbf{EN} & \textbf{ZH} & \textbf{All} \\ 
\midrule     
GPT-3.5-Turbo-16k   
& \textbf{45.1} & \textbf{36.2} & \underline{23.9} & \textbf{57.6} & \textbf{51.0} & 54.1                 
& \textbf{44.0} & \textbf{44.5} & \textbf{44.7} \\          
\midrule     
Llama2-7B-chat-4k    
& 21.7 & 18.2 & 18.5 & 49.9 & 4.1 & 48.1 
& 31.0 & 14.3 & 26.8 \\ 
LongChat-7B-32k  
& 28.8 & 20.3 & 22.5 & 50.8 & \underline{13.0} & 54.1       
& 34.3 & 23.9 & 31.6 \\ 
Vicuna-v1.5-7B-16k   
& \underline{31.8} & 18.8 & 23.2 & 56.8 & 5.3 & 47.3       
& 31.9 & \underline{26.4} & 30.5 \\
LongLoRA-7B-16k  
& 23.7 & 25.0 & 20.9 & 54.2 & 12.0 & 55.8  
& \underline{36.8} & 10.9 & 30.6 \\ 
\midrule 
\rowcolor{cyan!6}     
\textbf{HiCI-7B-16k}  
& 31.1$^{{\color{orange}\uparrow\textbf{7.4}}}$  
& \underline{26.8}$^{{\color{teal}\uparrow\text{1.8}}}$    
& 23.6$^{{\color{teal}\uparrow\text{2.7}}}$   
& \underline{57.1}$^{{\color{teal}\uparrow\text{2.9}}}$   
& 5.8       
& \underline{62.0}$^{{\color{teal}\uparrow\text{6.2}}}$  
&36.4&22.7$^{{\color{orange}\uparrow\textbf{11.8}}}$& \underline{33.2}$^{{\color{teal}\uparrow\text{2.6}}}$ \\ 
\rowcolor{cyan!6} 
\textbf{HiCI-7B-16k}$^\dagger$  
& 29.9$^{{\color{teal}\uparrow\text{6.2}}}$  
& 24.5     
& \textbf{24.6}$^{{\color{teal}\uparrow\text{3.7}}}$      
& 57.0$^{{\color{teal}\uparrow\text{2.8}}}$ 
&6.1&\textbf{63.8}$^{{\color{orange}\uparrow\textbf{8.0}}}$   
& 35.8   
& 23.4$^{{\color{orange}\uparrow\textbf{12.5}}}$ 
& 32.9$^{{\color{teal}\uparrow\text{2.3}}}$ \\ 
\bottomrule   
\addlinespace[1mm]        
\multicolumn{10}{l}  
{\footnotesize $^\dagger$ Applies training-consistent HiCI attention during inference prefill.}    
\end{tabular}  
\end{table*} 

\textbf{LongBench.} LongBench~\citep{bai2024longbench} is a bilingual benchmark comprising 21 tasks across six categories,    
with average input lengths of 5K--15K tokens. We perform context extension on RedPajama       
(4K$\to$16K) followed by instruction tuning on LongAlpaca-12k~\citep{chen2024longlora}, using LoRA~\citep{hu2022lora} with trainable embedding and normalization layers as in LongLoRA~\citep{chen2024longlora}. 
We evaluate two inference modes: HiCI with standard full attention, and HiCI$^\dagger$ which     
applies training-consistent hierarchical attention during prefill to reduce time-to-first-token latency. As shown in Table~\ref{tab:longbench}, HiCI outperforms LongLoRA across most categories, achieving 33.2\%       
overall (+2.6\%). The gains are particularly pronounced on Single-Document QA (\textbf{+7.4\%}) and Chinese tasks (\textbf{+11.8\%}),      
suggesting that the hierarchical inductive bias benefits both localized comprehension and cross-lingual         
transfer. HiCI$^\dagger$, despite using hierarchical attention during prefill, maintains competitive performance (32.9\%) and   
surpasses all baselines including the proprietary model GPT-3.5-Turbo-16K~\citep{achiam2023gpt} on both Summarization (24.6\%, +0.7\%) and Code (63.8\%,      
\textbf{+9.7\%}) tasks. This indicates that the learned hierarchical structure transfers robustly even under efficient inference.

\subsection{Ablation Studies}  
\label{sec:ablation}   

We systematically evaluate HiCI along two axes: component contribution and representation cardinality. All experiments use LLaMA-2-7B as the base model and evaluate under standard full attention. Segment-size sensitivity is analyzed in Appendix~\ref{app:training_loss}.

\begin{table}     
\centering                                    
\caption{Component and cardinality ablation for HiCI fine-tuned on LLaMA-2-7B at 8K context for 1{,}000 steps.}               
\label{tab:ablation}                             
\small                                           
\setlength{\tabcolsep}{5.5pt}                    
\begin{tabular}{lccc|cc|cc}                      
\toprule                                         
\multirow{2}{*}[-0.8ex]{\textbf{Variant}} &      
\multirow{2}{*}[-0.8ex]{\textbf{L}} &            
\multirow{2}{*}[-0.8ex]{\textbf{G}} &
\multirow{2}{*}[-0.8ex]{\textbf{B}} &            
\multicolumn{2}{c|}{\textbf{PG19}} &              
\multicolumn{2}{c}{\textbf{Proof-pile}} \\
\cmidrule{5-6} \cmidrule{7-8}                    
& & & & \textbf{4K} & \textbf{8K} & \textbf{4K} & \textbf{8K} \\                                          
\midrule                                          
HiCI & \textcolor{lightblue}{\ding{51}} & \textcolor{lightblue}{\ding{51}} & \textcolor{lightblue}{\ding{51}} &          
\textbf{7.01} & \textbf{6.93} & \textbf{2.82} & \textbf{2.65} \\  
w/o G & \textcolor{lightblue}{\ding{51}} & \textcolor{gray}{\ding{55}} & \textcolor{lightblue}{\ding{51}} & 7.25 &
7.04 & 2.95 & 2.78 \\                             
w/o L & \textcolor{gray}{\ding{55}} & \textcolor{lightblue}{\ding{51}} & \textcolor{lightblue}{\ding{51}} &
\underline{7.13} & \underline{6.99} & \underline{2.86} & \underline{2.69} \\                              
Only Group & \textcolor{gray}{\ding{55}} & \textcolor{gray}{\ding{55}} & \textcolor{gray}{\ding{55}} & 8.01 &
7.54 & 3.26 & 2.97 \\                           
\midrule        
$M=5,\ K=3$ & \textcolor{lightblue}{\ding{51}} & \textcolor{lightblue}{\ding{51}} & \textcolor{lightblue}{\ding{51}} &   
7.15 & 6.98 & 2.85 & 2.68 \\                     
$M=8,\ K=4$ & \textcolor{lightblue}{\ding{51}} & \textcolor{lightblue}{\ding{51}} & \textcolor{lightblue}{\ding{51}} &
\textbf{7.01} & \textbf{6.93} & \textbf{2.82} & \textbf{2.65} \\ 
$M=9,\ K=7$ & \textcolor{lightblue}{\ding{51}} & \textcolor{lightblue}{\ding{51}} & \textcolor{lightblue}{\ding{51}} &
\underline{7.10} & \underline{6.96} & \underline{2.86} & \underline{2.69} \\                              
\bottomrule     
\end{tabular}                                     
\end{table}

\textbf{Component and Cardinality Analysis.}
To quantify the contribution of each HiCI component, we train variants under   
8K context for 1{,}000 steps and evaluate on PG-19 and Proof-pile test sets. As shown in Table~\ref{tab:ablation}, removing global integration (w/o G) incurs nearly twice the degradation of removing local 
construction (w/o L), revealing that cross-segment aggregation contributes more substantially than 
within-segment compression.  
This asymmetry is corroborated by attention visualizations in Appendix~\ref{app:attention_analysis}.              
The Only Group baseline---grouped attention without hierarchical modules---yields markedly inferior 
performance, underscoring that explicit integration is indispensable beyond attention sparsification alone.  
For representation capacity, $(M{=}8, K{=}4)$ attains optimal performance, aligning with Miller's 
$7\pm2$ working memory bound~\citep{miller1956magical}; smaller capacities $(5, 3)$ prove insufficient, while
larger ones $(9, 7)$ compromise length generalization.  

\section{Limitations}
HiCI currently improves efficiency primarily during training;
autoregressive decoding still relies on standard full attention with a
growing KV cache, and extending the hierarchical mechanism to the
decoding phase remains an open direction.
A second limitation concerns models with very long native context
windows (e.g., 128K tokens or beyond). HiCI is primarily motivated by
settings where pre-training context is limited; whether it remains
beneficial for models already pre-trained at such lengths has not been
empirically studied.
Finally, improving long-context efficiency may lower the barrier to
large-scale handling of highly coherent long-form content, which could
be misused in adversarial settings or involve increasingly sensitive
user-provided data. These broader risks require system-level safeguards
beyond the scope of this work.

\section{Conclusion}         
\label{sec:conclusion} 

We presented HiCI, a hierarchical attention framework that decomposes long-context attention into local construction, global integration, and top-down broadcast. With only 4--5\% additional trainable parameters, HiCI substantially extends the effective context capacity of LLaMA-2 to 100K (7B) and 64K (13B), and of Qwen3-8B to 48K tokens, and we evaluate it across language modeling, retrieval, and downstream long-context benchmarks.
On PG-19 and Proof-pile, HiCI lowers perplexity relative to LongLoRA and ChunkLLaMA across all evaluated context lengths on LLaMA-2, and yields a $14\%$ perplexity reduction on PG-19 and an $8\%$ reduction on Proof-pile for Qwen3-8B at 48K. HiCI further attains $100\%$ passkey retrieval accuracy within the training range and degrades gracefully beyond it. On topic retrieval, HiCI matches GPT-4o-mini-128K and Claude-3.5-Sonnet-200K at perfect ($1.00$) retrieval accuracy through 13K, and reaches $0.94$ at 16K, outperforming all evaluated open-source baselines. On LongBench, HiCI surpasses GPT-3.5-Turbo-16K on both code comprehension (+9.7\%) and summarization. 
Ablations show that both hierarchical stages are necessary, with global integration the larger contributor. Overall, these findings suggest that the construction--integration principle provides an effective inductive bias for long-context modeling. Extending this hierarchical conditioning to vision--language models, where visual tokens significantly increase sequence length, is a promising direction for future work.

\section*{Acknowledgements} 
This work was supported in part by the Australian Research Council under Projects DP240101848 and FT230100549.
   
\section*{Impact Statement}
This paper presents work whose goal is to advance the field
of Machine Learning. There are many potential societal
consequences of our work, none which we feel must be
specifically highlighted here.

\clearpage
\bibliography{example_paper}

@inproceedings{vaswani2017attention,
    title={Attention is All You Need},
    author={Vaswani, Ashish and Shazeer, Noam and Parmar, Niki and Uszkoreit, Jakob and Jones, Llion and Gomez, Aidan N and Kaiser, {\L}ukasz and Polosukhin, Illia},
    booktitle={Advances in Neural Information Processing Systems},
    volume={30},
    year={2017}
  }

@inproceedings{brown2020language,
    title={Language Models are Few-Shot Learners},
    author={Brown, Tom and Mann, Benjamin and Ryder, Nick and Subbiah, Melanie and Kaplan, Jared D and Dhariwal, Prafulla and Neelakantan, Arvind and Shyam, Pranav and Sastry, Girish and Askell, Amanda and others},
    booktitle={Advances in Neural Information Processing Systems},
    volume={33},
    pages={1877--1901},
    year={2020}
  }

@article{liu2024lost,
    title={Lost in the Middle: How Language Models Use Long Contexts},
    author={Liu, Nelson F and Lin, Kevin and Hewitt, John and Paranjape, Ashwin and Bevilacqua, Michele and Petroni, Fabio and Liang, Percy},
    journal={Transactions of the Association for Computational Linguistics},
    volume={12},
    pages={157--173},
    year={2024}
  }

@article{chen2023extending,
    title={Extending Context Window of Large Language Models via Positional Interpolation},
    author={Chen, Shouyuan and Wong, Sherman and Chen, Liangjian and Tian, Yuandong},
    journal={arXiv preprint arXiv:2306.15595},
    year={2023}
  }

@inproceedings{bulatov2022recurrent,
    title={Recurrent Memory Transformer},
    author={Bulatov, Aydar and Kuratov, Yuri and Burtsev, Mikhail},
    booktitle={Advances in Neural Information Processing Systems},
    volume={35},
    year={2022}
  }

@article{munkhdalai2024leave,
    title={Leave No Context Behind: Efficient Infinite Context Transformers with Infini-attention},
    author={Munkhdalai, Tsendsuren and Faruqui, Manaal and Gopal, Siddharth},
    journal={arXiv preprint arXiv:2404.07143},
    year={2024}
  }

@article{beltagy2020longformer,
    title={Longformer: The Long-Document Transformer},
    author={Beltagy, Iz and Peters, Matthew E and Cohan, Arman},
    journal={arXiv preprint arXiv:2004.05150},
    year={2020}
  }

@inproceedings{zaheer2020bigbird,
    title={Big Bird: Transformers for Longer Sequences},
    author={Zaheer, Manzil and Guruganesh, Guru and Dubey, Kumar Avinava and Ainslie, Joshua and Alberti, Chris and Ontanon, Santiago and Pham, Philip and Ravula, Anirudh and Wang, Qifan and Yang, Li and Ahmed, Amr},
    booktitle={Advances in Neural Information Processing Systems},
    volume={33},
    year={2020}
  }

@article{kintsch1988role,
    title={The Role of Knowledge in Discourse Comprehension: A Construction-Integration Model},
    author={Kintsch, Walter},
    journal={Psychological Review},
    volume={95},
    number={2},
    pages={163--182},
    year={1988}
  }

@article{touvron2023llama2,
    title={Llama 2: Open Foundation and Fine-Tuned Chat Models},
    author={Touvron, Hugo and Martin, Louis and Stone, Kevin and others},
    journal={arXiv preprint arXiv:2307.09288},
    year={2023}
  }

@inproceedings{
dao2024flashattention,
title={FlashAttention-2: Faster Attention with Better Parallelism and Work Partitioning},
author={Tri Dao},
booktitle={The Twelfth International Conference on Learning Representations},
year={2024},
url={https://openreview.net/forum?id=mZn2Xyh9Ec}
}

@inproceedings{
mohtashami2023landmark,
title={Landmark Attention: Random-Access Infinite Context Length for Transformers},
author={Amirkeivan Mohtashami and Martin Jaggi},
booktitle={Workshop on Efficient Systems for Foundation Models @ ICML2023},
year={2023},
url={https://openreview.net/forum?id=PkoGERXS1B}
}

@inproceedings{he2025hmt,
  title={Hmt: Hierarchical memory transformer for efficient long context language processing},
  author={He, Zifan and Cao, Yingqi and Qin, Zongyue and Prakriya, Neha and Sun, Yizhou and Cong, Jason},
  booktitle={Proceedings of the 2025 Conference of the Nations of the Americas Chapter of the Association for Computational Linguistics: Human Language Technologies (Volume 1: Long Papers)},
  pages={8068--8089},
  year={2025}
}

@article{miller1956magical,
  title     = {The Magical Number Seven, Plus or Minus Two: Some Limits on Our Capacity for Processing Information},
  author    = {Miller, G. A.},
  journal   = {Psychological Review},
  volume    = {63},
  number    = {2},
  pages     = {81--97},
  year      = {1956},
  publisher = {American Psychological Association}
}

@article{cowan2001magical,
  title     = {The Magical Number 4 in Short-Term Memory: A Reconsideration of Mental Storage Capacity},
  author    = {Cowan, Nelson},
  journal   = {Behavioral and Brain Sciences},
  volume    = {24},
  number    = {1},
  pages     = {87--114},
  year      = {2001},
  publisher = {Cambridge University Press}
}

@article{together2023redpajama,
  title   = {RedPajama: An Open Dataset for Training Large Language Models},
  author  = {Together Computer},
  journal = {arXiv preprint arXiv:2307.09288},
  year    = {2023}
}

@inproceedings{rasley2020deepspeed,
  title={Deepspeed: System optimizations enable training deep learning models with over 100 billion parameters},
  author={Rasley, Jeff and Rajbhandari, Samyam and Ruwase, Olatunji and He, Yuxiong},
  booktitle={Proceedings of the 26th ACM SIGKDD international conference on knowledge discovery \& data mining},
  pages={3505--3506},
  year={2020}
}

@misc{azerbayev2022proofpile,
  title  = {Proof-Pile},
  author = {Azerbayev, Zhangir and Ayers, Edward and Piotrowski, Bartosz},
  year   = {2022},
  url    = {https://github.com/zhangir-azerbayev/proof-pile}
}

@inproceedings{
li2023how,
title={How Long Can Context Length of Open-Source {LLM}s truly Promise?},
author={Dacheng Li and Rulin Shao and Anze Xie and Ying Sheng and Lianmin Zheng and Joseph Gonzalez and Ion Stoica and Xuezhe Ma and Hao Zhang},
booktitle={NeurIPS 2023 Workshop on Instruction Tuning and Instruction Following},
year={2023},
url={https://openreview.net/forum?id=LywifFNXV5}
}

@article{DBLP:journals/corr/abs-2404-06654,
  publtype={informal},
  author={Cheng-Ping Hsieh and Simeng Sun and Samuel Kriman and Shantanu Acharya and Dima Rekesh and Fei Jia and Yang Zhang and Boris Ginsburg},
  title={RULER: What's the Real Context Size of Your Long-Context Language Models?},
  year={2024},
  cdate={1704067200000},
  journal={CoRR},
  volume={abs/2404.06654},
  url={https://doi.org/10.48550/arXiv.2404.06654}
}

@inproceedings{
chen2024longlora,
title={LongLo{RA}: Efficient Fine-tuning of Long-Context Large Language Models},
author={Yukang Chen and Shengju Qian and Haotian Tang and Xin Lai and Zhijian Liu and Song Han and Jiaya Jia},
booktitle={The Twelfth International Conference on Learning Representations},
year={2024},
url={https://openreview.net/forum?id=6PmJoRfdaK}
}

@article{tworkowski2023focused,
  title={Focused transformer: Contrastive training for context scaling},
  author={Tworkowski, Szymon and Staniszewski, Konrad and Pacek, Miko{\l}aj and Wu, Yuhuai and Michalewski, Henryk and Mi{\l}o{\'s}, Piotr},
  journal={Advances in neural information processing systems},
  volume={36},
  pages={42661--42688},
  year={2023}
}

@inproceedings{tishby2000information,
  title={The information bottleneck method},
  author={Tishby, Naftali and Pereira, Fernando C and Bialek, William},
  booktitle={Proceedings of the 37th Annual Allerton Conference on Communication, Control, and Computing},
  pages={368--377},
  year={2000}
}

@article{du2022glm,
    title={GLM: General Language Model Pretraining with Autoregressive Blank Infilling},
    author={Du, Zhengxiao and Qian, Yujie and Liu, Xiao and Ding, Ming and Qiu, Jiezhong and Yang, Zhilin and Tang, Jie},
    journal={arXiv preprint arXiv:2103.10360},
    year={2022}
  }

@misc{MosaicML2023mpt,
    title={Introducing MPT-7B: A New Standard for Open-Source, Commercially Usable LLMs},
    author={MosaicML},
    year={2023},
    howpublished={\url{https://www.mosaicml.com/blog/mpt-7b}}
  }

@inproceedings{dai2019transformerxl,
  title        = {Transformer-XL: Attentive Language Models Beyond a Fixed-Length Context},
  author       = {Dai, Zihang and Yang, Zhilin and Yang, Yiming and Carbonell, Jaime and Le, Quoc V. and Salakhutdinov, Ruslan},
  booktitle    = {Proceedings of the 57th Annual Meeting of the Association for Computational Linguistics (ACL)},
  year         = {2019},
  pages        = {2978--2988},
  publisher    = {Association for Computational Linguistics},
  url          = {https://aclanthology.org/P19-1285.pdf}
}

@inproceedings{
Rae2020Compressive,
title={Compressive Transformers for Long-Range Sequence Modelling},
author={Jack W. Rae and Anna Potapenko and Siddhant M. Jayakumar and Chloe Hillier and Timothy P. Lillicrap},
booktitle={International Conference on Learning Representations},
year={2020},
url={https://openreview.net/forum?id=SylKikSYDH}
}

@inproceedings{katharopoulos2020transformers,
  title        = {Transformers are RNNs: Fast Autoregressive Transformers with Linear Attention},
  author       = {Katharopoulos, Anthony and Vyas, Apoorv and Pappas, Nikolaos and Fleuret, Fran{\c{c}}ois},
  booktitle    = {International Conference on Machine Learning (ICML)},
  year         = {2020},
  pages        = {5156--5165},
  publisher    = {PMLR}
}

@inproceedings{arora2024zoology,
  title={Zoology: Measuring and improving recall in efficient language models},
  author={Arora, Simran and Eyuboglu, Sabri and Timalsina, Aman and Johnson, Isys and Poli, Michael and Zou, James Y and Rudra, Atri and R{\'e}, Christopher},
  booktitle={International conference on learning representations},
  volume={2024},
  pages={15664--15730},
  year={2024}
}

@inproceedings{zhu2024pose,
    title={Po{SE}: Efficient Context Window Extension of {LLM}s via Positional Skip-wise Training},
    author={Zhu, Dawei and Yang, Nan and Wang, Liang and Song, Yifan and Wu, Wenhao and Wei, Furu and Li, Sujian},
    booktitle={The Twelfth International Conference on Learning Representations},
    year={2024},
    url={https://openreview.net/forum?id=3Z1gxuAQrA}
  }

@inproceedings{
ding2024longrope,
title={LongRo{PE}: Extending {LLM} Context Window Beyond 2 Million Tokens},
author={Yiran Ding and Li Lyna Zhang and Chengruidong Zhang and Yuanyuan Xu and Ning Shang and Jiahang Xu and Fan Yang and Mao Yang},
booktitle={Forty-first International Conference on Machine Learning},
year={2024},
url={https://openreview.net/forum?id=ONOtpXLqqw}
}

@inproceedings{bai2024longalign,
    title={{L}ong{A}lign: A Recipe for Long Context Alignment of Large Language Models},
    author={Bai, Yushi and Lv, Xin and Zhang, Jiajie and He, Yuze and Qi, Ji and Hou, Lei and Tang, Jie and Dong, Yuxiao and Li, Juanzi},
    booktitle={Findings of the Association for Computational Linguistics: EMNLP 2024},
    pages={1376--1395},
    year={2024},
    publisher={Association for Computational Linguistics}
  }

@inproceedings{peng2024yarn,
    title={Ya{RN}: Efficient Context Window Extension of Large Language Models},
    author={Peng, Bowen and Quesnelle, Jeffrey and Fan, Honglu and Shippole, Enrico},
    booktitle={The Twelfth International Conference on Learning Representations},
    year={2024},
    url={https://openreview.net/forum?id=wHBfxhZu1u}
  }

@book{kintsch1998comprehension,
    title={Comprehension: A Paradigm for Cognition},
    author={Kintsch, Walter},
    year={1998},
    publisher={Cambridge University Press}
  }

@article{dehaene2001towards,
    title={Towards a Cognitive Neuroscience of Consciousness: Basic Evidence and a Workspace Framework},
    author={Dehaene, Stanislas and Naccache, Lionel},
    journal={Cognition},
    volume={79},
    number={1-2},
    pages={1--37},
    year={2001},
    publisher={Elsevier}
  }

@article{felleman1991distributed,
    title={Distributed Hierarchical Processing in the Primate Cerebral Cortex},
    author={Felleman, Daniel J and Van Essen, David C},
    journal={Cerebral Cortex},
    volume={1},
    number={1},
    pages={1--47},
    year={1991},
    publisher={Oxford University Press}
  }

@inproceedings{hutchins2022block, title={Block-Recurrent Transformers}, author={Hutchins, DeLesley and Schlag,   
  Imanol and Wu, Yuhuai and Dyer, Ethan and Neyshabur, Behnam}, booktitle={Advances in Neural Information          
  Processing Systems}, year={2022}}

@inproceedings{ho2024block, title={Block Transformer: Global-to-Local Language Modeling for Fast Inference},     
  author={Ho, Namgyu and Bae, Sangmin and Kim, Taehyeon and Jo, Hyunjik and Kim, Yireun and Schuster, Tal and      
  Fisch, Adam and Thorne, James and Yun, Se-Young}, booktitle={Advances in Neural Information Processing Systems}, 
  year={2024}}

@inproceedings{fountas2025emllm, title={Human-inspired Episodic Memory for Infinite Context {LLM}s},             
  author={Fountas, Zafeirios and Benfeghoul, Martin A and Oomerjee, Adnan and Christopoulou, Fenia and Lampouras,  
  Gerasimos and Bou-Ammar, Haitham and Wang, Jun}, booktitle={International Conference on Learning                 
Representations}, year={2025}}

@inproceedings{bai2024longbench,
  title={Longbench: A bilingual, multitask benchmark for long context understanding},
  author={Bai, Yushi and Lv, Xin and Zhang, Jiajie and Lyu, Hongchang and Tang, Jiankai and Huang, Zhidian and Du, Zhengxiao and Liu, Xiao and Zeng, Aohan and Hou, Lei and others},
  booktitle={Proceedings of the 62nd annual meeting of the association for computational linguistics (volume 1: Long papers)},
  pages={3119--3137},
  year={2024}
}

@inproceedings{hu2022lora, 
  title = "{LoRA}: Low-Rank Adaptation of Large Language Models",  
  author = "Hu, Edward J. and Shen, Yelong and Wallis, Phillip and Allen-Zhu, Zeyuan and Li, Yuanzhi and
Wang, Shean and Wang, Lu and Chen, Weizhu", 
  booktitle = "International Conference on Learning Representations (ICLR)", 
  year = "2022",   
  url = "https://openreview.net/forum?id=nZeVKeeFYf9",  
}

@inproceedings{rajbhandari2020zero,
  title={Zero: Memory optimizations toward training trillion parameter models},
  author={Rajbhandari, Samyam and Rasley, Jeff and Ruwase, Olatunji and He, Yuxiong},
  booktitle={SC20: International Conference for High Performance Computing, Networking, Storage and Analysis},
  pages={1--16},
  year={2020},
  organization={IEEE}
}

@article{achiam2023gpt,
  title={Gpt-4 technical report},
  author={Achiam, Josh and Adler, Steven and Agarwal, Sandhini and Ahmad, Lama and Akkaya, Ilge and Aleman, Florencia Leoni and Almeida, Diogo and Altenschmidt, Janko and Altman, Sam and Anadkat, Shyamal and others},
  journal={arXiv preprint arXiv:2303.08774},
  year={2023}
}

@article{DBLP:journals/corr/abs-2108-12409,
  publtype={informal},
  author={Ofir Press and Noah A. Smith and Mike Lewis},
  title={Train Short, Test Long: Attention with Linear Biases Enables Input Length Extrapolation},
  year={2021},
  cdate={1609459200000},
  journal={CoRR},
  volume={abs/2108.12409},
  url={https://arxiv.org/abs/2108.12409}
}

@article{DBLP:journals/corr/abs-2307-02486,
  publtype={informal},
  author={Jiayu Ding and Shuming Ma and Li Dong and Xingxing Zhang and Shaohan Huang and Wenhui Wang and Nanning Zheng and Furu Wei},
  title={LongNet: Scaling Transformers to 1, 000, 000, 000 Tokens},
  year={2023},
  cdate={1672531200000},
  journal={CoRR},
  volume={abs/2307.02486},
  url={https://doi.org/10.48550/arXiv.2307.02486},
}

@book{baars1993cognitive,
  title={A cognitive theory of consciousness},
  author={Baars, Bernard J},
  year={1993},
  publisher={Cambridge University Press}
}

@article{vanrullen2021deep,
  title={Deep learning and the global workspace theory},
  author={VanRullen, Rufin and Kanai, Ryota},
  journal={Trends in Neurosciences},
  volume={44},
  number={9},
  pages={692--704},
  year={2021},
  publisher={Elsevier}
}

@inproceedings{DBLP_conf_iclr_GoyalDLBKRBBMB22,
  author={Anirudh Goyal and Aniket Rajiv Didolkar and Alex Lamb and Kartikeya Badola and Nan Rosemary Ke and Nasim Rahaman and Jonathan Binas and Charles Blundell and Michael Curtis Mozer and Yoshua Bengio},
  title={Coordination Among Neural Modules Through a Shared Global Workspace},
  year={2022},
  cdate={1640995200000},
  url={https://openreview.net/forum?id=XzTtHjgPDsT},
  booktitle={ICLR},
}

@inproceedings{hong2024concept,
  title={Concept-centric transformers: Enhancing model interpretability through object-centric concept learning within a shared global workspace},
  author={Hong, Jinyung and Park, Keun Hee and Pavlic, Theodore P},
  booktitle={Proceedings of the IEEE/CVF Winter Conference on Applications of Computer Vision},
  pages={4880--4891},
  year={2024}
}

@article{barrault2024large,
  title={Large concept models: Language modeling in a sentence representation space},
  author={Barrault, Lo{\"\i}c and Duquenne, Paul-Ambroise and Elbayad, Maha and Kozhevnikov, Artyom and Alastruey, Belen and Andrews, Pierre and Coria, Mariano and Couairon, Guillaume and Costa-juss{\`a}, Marta R and Dale, David and others},
  journal={arXiv preprint arXiv:2412.08821},
  year={2024}
}

@inproceedings{
zeng2024a,
title={A Framework for Inference Inspired by Human Memory Mechanisms},
author={Xiangyu Zeng and Jie Lin and Piao Hu and Ruizheng Huang and Zhicheng Zhang},
booktitle={The Twelfth International Conference on Learning Representations},
year={2024},
url={https://openreview.net/forum?id=vBo7544jZx}
}

@inproceedings{
hao2025training,
title={Training Large Language Models to Reason in a Continuous Latent Space},
author={Shibo Hao and Sainbayar Sukhbaatar and DiJia Su and Xian Li and Zhiting Hu and Jason E Weston and Yuandong Tian},
booktitle={Workshop on Reasoning and Planning for Large Language Models},
year={2025},
url={https://openreview.net/forum?id=KrWSrrYGpT}
}

@inproceedings{li2025identify,
  title={Identify, Isolate, and Purge: Mitigating Hallucinations in LVLMs via Self-Evolving Distillation},
  author={Li, Wenhao and Su, Xiu and Wu, Jingyi and Yang, Feng and Liu, Yang and Chen, Yi and You, Shan and Xu, Chang},
  booktitle={Proceedings of the 33rd ACM International Conference on Multimedia},
  pages={6791--6800},
  year={2025}
}

@inproceedings{xu2026rethinking,
  title={Rethinking visual token reduction in lvlms under cross-modal misalignment},
  author={Xu, Rui and Wang, Yunke and Luo, Yong and Du, Bo},
  booktitle={Proceedings of the AAAI Conference on Artificial Intelligence},
  volume={40},
  number={32},
  pages={27323--27331},
  year={2026}
}

@inproceedings{
wang2025position,
title={Position: {AI} Scaling: From Up to Down and Out},
author={Yunke Wang and Yanxi Li and Chang Xu},
booktitle={Forty-second International Conference on Machine Learning Position Paper Track},
year={2025},
url={https://openreview.net/forum?id=UTxi86wmas}
}

@article{xu2026vla,
  title={Vla-cache: Efficient vision-language-action manipulation via adaptive token caching},
  author={Xu, Siyu and Wang, Yunke and Xia, Chenghao and Zhu, Dihao and Huang, Tao and Xu, Chang},
  journal={Advances in Neural Information Processing Systems},
  volume={38},
  pages={164448--164473},
  year={2026}
}

@article{hurst2024gpt,
  title={Gpt-4o system card},
  author={Hurst, Aaron and Lerer, Adam and Goucher, Adam P and Perelman, Adam and Ramesh, Aditya and Clark, Aidan and Ostrow, AJ and Welihinda, Akila and Hayes, Alan and Radford, Alec and others},
  journal={arXiv preprint arXiv:2410.21276},
  year={2024}
}

@article{anthropic2024claude,
  title={The claude 3 model family: Opus, sonnet, haiku},
  author={Anthropic, AI},
  journal={Claude-3 Model Card},
  volume={1},
  number={1},
  pages={4},
  year={2024}
}

@article{grattafiori2024llama,
  title={The llama 3 herd of models},
  author={Grattafiori, Aaron and Dubey, Abhimanyu and Jauhri, Abhinav and Pandey, Abhinav and Kadian, Abhishek and Al-Dahle, Ahmad and Letman, Aiesha and Mathur, Akhil and Schelten, Alan and Vaughan, Alex and others},
  journal={arXiv preprint arXiv:2407.21783},
  year={2024}
}

@inproceedings{DBLP_conf_icml_An0ZGQZK24,
  author={Chenxin An and Fei Huang and Jun Zhang and Shansan Gong and Xipeng Qiu and Chang Zhou and Lingpeng Kong},
  title={Training-Free Long-Context Scaling of Large Language Models},
  year={2024},
  cdate={1704067200000},
  url={https://openreview.net/forum?id=If4xW9vF7U},
  booktitle={ICML},
}

@article{yang2025qwen3,
  title={Qwen3 technical report},
  author={Yang, An and Li, Anfeng and Yang, Baosong and Zhang, Beichen and Hui, Binyuan and Zheng, Bo and Yu, Bowen and Gao, Chang and Huang, Chengen and Lv, Chenxu and others},
  journal={arXiv preprint arXiv:2505.09388},
  year={2025}
}

@article{wang2026sibyl,
  title={Sibyl-AutoResearch: Autonomous Research Needs Self-Evolving Trial-and-Error Harnesses, Not Paper Generators},
  author={Wang, Chengcheng and Xie, Qinhua and He, Wei and Guo, Jianyuan and Wang, Shiqi and Xu, Chang},
  journal={arXiv preprint arXiv:2605.22343},
  year={2026}
}
\bibliographystyle{icml2026}

\newpage
\appendix
\onecolumn

\section{Theoretical Analysis}
  \label{app:theory}

  This appendix provides theoretical analysis of HiCI's architectural choices.
  Rather than establishing optimality, our goal is to characterize the
  information-theoretic and computational properties that underlie the
  empirical behaviors observed in experiments: the effectiveness of compact
  representations, the role of shared compression, and the trade-offs inherent
  in fixed-capacity hierarchical integration.

  \subsection{Notation}
  \label{app:notation}

  Let $X \in \mathbb{R}^{T \times d}$ denote an input sequence of $T$ tokens
  with hidden dimension $d$. HiCI partitions $X$ into $N = T/S$ non-overlapping
  segments $\{X_i\}_{i=1}^{N}$, each of length $S$.
  Table~\ref{tab:notation} summarizes the key architectural hyperparameters.
  All are fixed constants chosen before training and remain invariant across
  sequence lengths at inference time.

  \begin{table}[h]
  \centering
  \caption{Summary of notation.}
  \label{tab:notation}
  \small
  \begin{tabular}{cl}
  \toprule
  \textbf{Symbol} & \textbf{Description} \\
  \midrule
  $M$ & Local cardinality (queries per segment) \\
  $K$ & Global cardinality (context vectors) \\
  $d_s$ & Intermediate compression dimension \\
  $d_b$ & Bottleneck dimension for attention \\
  \bottomrule
  \end{tabular}
  \end{table}

  \subsection{Hierarchical Information Flow}
  \label{app:hierarchy}

  We formalize HiCI's hierarchical structure through functional decomposition
  and analyze the resulting information flow.

  \paragraph{Compositional Structure.}
  A HiCI block computes the output through three composed functions:
  \begin{equation}
  L = f_{\text{local}}(X), \qquad
  G = f_{\text{global}}(L), \qquad
  \tilde{X} = f_{\text{broadcast}}(X, L, G),
  \end{equation}
  where $L = \{L_i\}_{i=1}^{N}$ with $L_i \in \mathbb{R}^{M \times d}$ denotes
  the local representations extracted from each segment, and
  $G \in \mathbb{R}^{K \times d}$ denotes the global context aggregated
  from all segments. This decomposition directly mirrors the three computational
  stages described in \S\ref{sec:overview}.

  \paragraph{Cross-Segment Dependency.}
  Consider two tokens $x_s \in X_j$ and $x_t \in X_i$ residing in different
  segments ($j \neq i$). Under standard segmented attention, these tokens
  cannot interact since attention is restricted within each segment.
  HiCI overcomes this limitation by introducing a hierarchical pathway:
  \begin{equation}
  x_s \;\longrightarrow\; L_j \;\longrightarrow\; G \;\longrightarrow\; \tilde{x}_t.
  \end{equation}
  This three-hop path enables sequence-wide information flow while preserving
  the computational benefits of segment-parallel processing.

  \paragraph{Receptive Field.}
  In the broadcast stage, each token $x_t \in X_i$ attends over the augmented
  context $[G; L_i; X_i] \in \mathbb{R}^{(K+M+S) \times d}$. The attention
  output takes the form:
  \begin{equation}
  \tilde{x}_t = \sum_{j=1}^{K+M+S} \alpha_{tj} \, v_j,
  \end{equation}
  where $\{v_j\}$ are value projections and $\{\alpha_{tj}\}$ are
  softmax-normalized attention weights computed jointly over all $K+M+S$ positions.
  Since the global context $G$ aggregates information from all $N$ segments,
  each token gains indirect access to the entire sequence through the first
  $K$ positions of the augmented context.

\subsection{Causality of Global Context Construction}
\label{app:causality}
The Global Integration step in \S\ref{sec:global} aggregates all local
representations $\{L_i\}_{i=1}^{N}$ into a shared global
context $G \in \mathbb{R}^{K \times d}$, which is prepended to each segment's key--value sequence. Although token-to-token attention within each segment remains strictly causal, the shared prefix $G$ introduces a cross-segment non-causal pathway by aggregating representations derived from future segments. To isolate the contribution of this pathway, we replace
$G$ with a strictly causal counterpart
\[
G_i \;=\; \mathrm{Agg}(L_1, \dots, L_i),
\]
in which segment $i$ aggregates only from current and
preceding local representations. All other components of
HiCI remain unchanged.

\begin{table}[h]
\centering
\caption{Perplexity ($\downarrow$) on the PG-19 test set for LLaMA-2-7B at evaluation context lengths of ${2\text{K}, 4\text{K}, 8\text{K}}$. All models are fine-tuned for 1000 steps using HiCI with segment size $S{=}2048$, and evaluated under standard full attention, consistent with Table~\ref{tab:ppl}.}
\label{tab:causal-ppl}
\small
\begin{tabular}{lccc}
\toprule
 & 2K & 4K & 8K \\
\midrule
LongLoRA      & 7.70 & 7.35 & 7.14 \\
Shared $G$    & 7.27 & 7.01 & 6.93 \\
Causal $G_i$  & 7.28 & 7.00 & 6.94 \\
\bottomrule
\end{tabular}
\end{table}

Across all evaluation lengths, $G$ and $G_i$ yield nearly identical perplexity
($|\Delta\mathrm{PPL}| = 0.01$),
while both improve over LongLoRA by
$0.21$--$0.43$. These results suggest that HiCI's gains are not primarily driven by cross-segment future access,
but by the hierarchical Construction--Integration--Broadcast structure.

  \subsection{Cardinality Design}
  \label{app:cardinality}

  The cardinalities $M$ and $K$ govern the capacity of local and global
  representations, respectively. Here we discuss their design rationale.

  \paragraph{Cognitive Motivation.}
  Following theories of limited working memory capacity~\citep{miller1956magical,cowan2001magical},
  we constrain both $M$ and $K$ to small constants independent of sequence length $T$.
  This fixed-capacity bottleneck encourages the model to learn hierarchical
  abstractions rather than relying on token-level memorization.

  \paragraph{Local Cardinality ($M$).}
  The parameter $M$ determines the number of learnable queries used in local
  construction, and hence the dimensionality of each local representation
  $L_i \in \mathbb{R}^{M \times d}$. Empirically, we observe that larger $M$
  improves performance at the training context length but degrades generalization
  to shorter sequences. This behavior is consistent with overfitting to
  length-specific patterns when excess capacity is available.
  We set $M = 8$ to balance in-distribution accuracy and length robustness.

\paragraph{Global Cardinality ($K$).}    
The parameter $K$ determines the dimensionality of the global context        
$G \in \mathbb{R}^{K \times d}$. Unlike $M$, the global integration stage    
operates on a fixed-size input (five statistical summaries), rendering $K$   
inherently decoupled from sequence length. We set $K = 4$; the attention-based weighting learns to project the  
five statistical views into $K$ compact global slots.

  \subsection{Local Compression}
  \label{app:local-capacity}

  The local construction stage maps each segment $X_i \in \mathbb{R}^{S \times d}$
  to a compact representation $L_i \in \mathbb{R}^{M \times d}$ with $M \ll S$.
  Motivated by cognitive theories of limited working memory (\S\ref{sec:local}),
  we fix $M$ as a small constant and analyze the information-theoretic
  implications of this design.

  \paragraph{Capacity Bound.}
  The cross-attention mechanism projects keys and values into a $d_b$-dimensional
  subspace before aggregation. Under a standard linear-Gaussian approximation---treating
  the bottleneck projection as an information channel with effective signal variance
  $\sigma_X^2$ and noise variance $\sigma_\epsilon^2$---the mutual information
  between a segment and its local representation admits the capacity-style bound:
  \begin{equation}
  I(X_i; L_i) \;\lesssim\; M \cdot d_b \cdot \log\!\left(1 + \frac{\sigma_X^2}{\sigma_\epsilon^2}\right).
  \end{equation}
  This bound highlights that the representational budget scales with the product
  $M \cdot d_b$, not with segment length $S$. While not a tight guarantee for
  attention in general, it provides a useful characterization of how $(M, d_b)$
  jointly control the information throughput of the local interface.

  \paragraph{Inductive Bias.}
  The fixed bottleneck $(M, d_b)$ forces the model to compress each segment into
  a small set of salient factors, functioning as an inductive bias toward
  abstraction. Fine-grained token details must compete for a limited representational
  budget, favoring task-relevant structure. The capacity--generalization trade-off
  discussed in \S\ref{app:cardinality} follows directly from this constraint.

\subsection{Statistical Aggregation}   
\label{app:stats}      
The global integration stage aggregates all local representations  
$\mathcal{L} \in \mathbb{R}^{(NM) \times d}$ into a fixed-size summary 
$\mathbf{Z} \in \mathbb{R}^{5 \times d}$ through five complementary statistics.                                 
Table~\ref{tab:stats} describes the information captured by each statistic.      
\begin{table}[h]                      
\centering                            
\caption{Statistical summaries computed in global integration.}   
\label{tab:stats}                     
\small                                 
\begin{tabular}{ll}                   
\toprule                               
\textbf{Statistic} & \textbf{Captured Information} \\    
\midrule                              
$\boldsymbol{\mu}$ (mean) & Central tendency \\  
$\boldsymbol{\sigma}$ (std) & Dispersion \\  
$\boldsymbol{\mu}^{+}, \boldsymbol{\mu}^{-}$ (max, min) & Extremal activations \\                               
$\hat{\boldsymbol{\mu}}$ (normalized mean) & Directional structure \\  
\bottomrule                           
\end{tabular}                        
\end{table}    
Together, these statistics provide a coarse characterization of the local                                       
representation distribution without retaining individual identities.   
\paragraph{Fixed-Size Interface.}      
A key property of this design is that the intermediate summary $\mathbf{Z} \in \mathbb{R}^{5 \times d}$         
remains constant regardless of sequence length $T$ or the number of segments $N$.                               
The subsequent attention-based weighting (\S\ref{sec:global}) then projects $\mathbf{Z}$                        
into the final global context $G \in \mathbb{R}^{K \times d}$ with $K = 4$ slots.                               
This two-stage process decouples global context capacity from sequence length,                                  
enabling the same architecture to operate across varying context sizes 
(see \S\ref{app:cardinality} for ablations on $K$).   

  \subsection{Two-Stage Compression}
  \label{app:compression}

  The shared compression $\phi\colon \mathbb{R}^d \to \mathbb{R}^{d_b}$
  proceeds through an intermediate bottleneck dimension:
  \begin{equation}
  \phi = \psi_b \circ \psi_c\colon \quad
  \mathbb{R}^d \xrightarrow{\;\psi_c\;} \mathbb{R}^{d_s} \xrightarrow{\;\psi_b\;} \mathbb{R}^{d_b},
  \end{equation}
  with $d_s < d_b \ll d$ ($d_s = 128$, $d_b = 512$, $d = 4096$ in our experiments).

  \paragraph{Regularization via Bottleneck.}
  The intermediate dimension $d_s$ imposes a capacity constraint before expansion
  to $d_b$. By the data processing inequality, information in the final
  representation is bounded by what passes through the narrower bottleneck:
  \begin{equation}
  I(\mathbf{Z}; \phi(\mathbf{Z})) \;\leq\; I(\mathbf{Z}; \psi_c(\mathbf{Z})).
  \end{equation}
  This two-stage design forces the model to first identify a compact, task-relevant
  subspace before expanding to the attention dimension.

  \paragraph{View Invariance.}
  Applying identical compression parameters to all five statistics enforces
  \emph{view-invariant} encoding: the model must learn a common projection that
  preserves relevant information across heterogeneous statistical views.
  This acts as structural regularization, encouraging consistent representations
  rather than view-specific overfitting. Our ablations (\S\ref{sec:ablation})
  confirm that using separate projections per view yields marginal or no improvement,
  validating the shared bottleneck design.

  \subsection{Computational Complexity}
  \label{app:complexity}

  We analyze the computational complexity of HiCI and establish its linear
  scaling with respect to sequence length.
  
  \begin{theorem}[Linear Complexity]
  \label{thm:complexity}
  HiCI achieves time complexity $O(TSd)$ and space complexity $O(S^2)$ per layer,
  linear in $T$ for fixed $S$. An additional $O((K{+}M)d)$ space is required
  for storing the hierarchical context, which is negligible for typical
  configurations ($K{+}M = 12$, $S \geq 1024$).
  \end{theorem}

  \begin{proof}
  Let $N = T/S$ denote the number of segments. We analyze each stage separately.

  \paragraph{Local Construction.}
  For each segment, cross-attention between $M$ learnable queries and $S$ segment
  tokens incurs:
  \begin{equation}
  \underbrace{(M + S) \cdot d \cdot d_b}_{\text{projections}}
  \;+\;
  \underbrace{M \cdot S \cdot d_b}_{\text{attention}}
  \;+\;
  \underbrace{M \cdot d_b \cdot d}_{\text{output}}
  \;=\; O(S \cdot d \cdot d_b),
  \end{equation}
  where the dominant cost arises from key-value projections over $S$ tokens.
  Aggregating over $N$ segments yields a total cost of $O(T \cdot d \cdot d_b)$.

  \paragraph{Global Integration.}
  Computing statistical summaries over all $NM$ local vectors requires $O(NMd) = O(Td/S)$.
  The subsequent two-stage compression and global attention operate on fixed-size
  inputs (5 statistics and $K$ queries), contributing $O(1)$ with respect to $T$.

  \paragraph{Top-down Broadcast.}
  Each segment attends over an augmented context of size $(K + M + S)$:
  \begin{equation}
  \underbrace{(K + M + S) \cdot d^2}_{\text{projections}}
  \;+\;
  \underbrace{S \cdot (K + M + S) \cdot d}_{\text{attention}}
  \;=\; O(S^2 \cdot d + S \cdot d^2),
  \end{equation}
  where the quadratic dependence on $S$ dominates for typical hidden dimensions.
  Summing over $N$ segments gives a total cost of $O(T \cdot S \cdot d)$.

  \paragraph{Overall Complexity.}
  Combining all stages:
  \begin{equation}
  O(T \cdot d \cdot d_b) + O(T \cdot d / S) + O(T \cdot S \cdot d) = O(T \cdot S \cdot d),
  \end{equation}
  where the broadcast stage is asymptotically dominant.
  For fixed $S$, the overall time complexity is linear in $T$.
  \end{proof}

  Table~\ref{tab:complexity} compares HiCI's complexity with related methods.
  HiCI retains the $O(TSd)$ time complexity of windowed attention while
  introducing explicit hierarchical cross-segment interactions.

  \begin{table}[h]
  \centering
  \caption{Computational complexity comparison. $T$: sequence length,
  $d$: hidden dimension, $S$: segment/window size.}
  \label{tab:complexity}
  \small
  \begin{tabular}{lccc}
  \toprule
  \textbf{Method} & \textbf{Time} & \textbf{Space} & \textbf{Cross-Segment} \\
  \midrule
  Standard Attention & $O(T^2 d)$ & $O(T^2)$ & Full \\
  Linear Attention & $O(Td^2)$ & $O(Td)$ & Approximated \\
  \midrule
  Segmented Attention & $O(TSd)$ & $O(S^2)$ & None \\
  LongLoRA & $O(TSd)$ & $O(S^2)$ & Shifted windows \\
  \textbf{HiCI} & $O(TSd)$ & $O(S^2)$ & Hierarchical \\
  \bottomrule
  \end{tabular}
  \end{table}

  \begin{remark}
  While HiCI and segmented attention share the same asymptotic complexity,
  HiCI incurs a small constant overhead from the $(K + M)$ additional context
  tokens per segment. With $K = 4$ and $M = 8$, this overhead is negligible
  for typical segment sizes ($S \geq 1024$), adding less than 2\% to the
  attention computation while enabling cross-segment information flow.
  \end{remark}

\section{Training Details} 
\label{app:training}  

\subsection{Hyperparameters}
\label{app:hyperparameters}  

Table~\ref{tab:hyperparameters} summarises the hyperparameters used for HiCI training.
Unless otherwise specified, the same configuration is adopted for context lengths
from 8K to 64K for the 7B model and from 8K to 32K for the 13B model.
For the maximum-length settings (100K for 7B and 64K for 13B), we employ
DeepSpeed ZeRO Stage-3 and increase the number of segments to $N{=}10$ and $N{=}8$,
respectively, to satisfy memory constraints.
Supervised fine-tuning on LongAlpaca-12k is performed for 5 epochs; all remaining
hyperparameters are held constant.

\begin{table}[t]
\centering
\caption{Hyperparameters for HiCI training. PT denotes continued pretraining on RedPajama, and SFT denotes supervised fine-tuning on LongAlpaca-12k.} 
\label{tab:hyperparameters}
\small
\setlength{\tabcolsep}{3.5pt}
\renewcommand{\arraystretch}{1}

\begin{tabular}{lcccc} 
\toprule & \multicolumn{3}{c}{\textbf{LLaMA-2}} & \textbf{Qwen3} \\                  
\cmidrule(lr){2-4} \cmidrule(lr){5-5} 
\multirow{-2}{*}{\textbf{Hyperparameter}} & \textbf{7B (PT)} & \textbf{13B (PT)} & \textbf{7B (SFT)} & \textbf{8B (PT)} \\                            

\midrule
\multicolumn{5}{l}{\textbf{Optimization}} \\
Optimizer & AdamW & AdamW & AdamW & AdamW \\
Backbone learning rate & $2 \times 10^{-5}$ & $2 \times 10^{-5}$ & $2 \times 10^{-5}$ & $2 \times 10^{-5}$ \\
HiCI learning rate & $2 \times 10^{-4}$ & $2 \times 10^{-4}$ & $2 \times 10^{-4}$ & $2 \times 10^{-4}$ \\
Weight decay & 0 & 0 & 0 & 0 \\
LR scheduler & Constant w/ warmup & Constant w/ warmup & Constant w/ warmup & Constant w/ warmup \\
Warmup steps & 20 & 20 & 20 & 20 \\
Training duration & 1{,}000 steps & 1{,}000 steps  & 5 epochs & 1{,}000 steps \\

\midrule
\addlinespace[3pt]
\multicolumn{5}{l}{\textbf{Batch Configuration}} \\
Per-device batch size & 1 & 1 & 1 & 1 \\
Gradient accumulation & 8 & 8 & 8 & 8 \\
Number of GPUs & 8 & 8 & 8 & 8 \\
Effective batch size & 64 & 64 & 64 & 64 \\

\midrule
\addlinespace[3pt]
\multicolumn{5}{l}{\textbf{LoRA}} \\
LoRA rank $r$ & 8 & 8 & 8 & 8 \\
LoRA alpha $\alpha$ & 16 & 16 & 16 & 16 \\
LoRA dropout & 0.05 & 0.05 & 0.05 & 0.05 \\

\midrule
\addlinespace[3pt]
\multicolumn{5}{l}{\textbf{HiCI Architecture}} \\
Number of segments $N$ & 4 & 4 & 4 & 4 \\
Local slots $M$ & 8 & 8 & 8 & 8 \\
Global slots $K$ & 4 & 4 & 4 & 4 \\
Attention heads & 8 & 10 & 8 & 8 \\
Bottleneck dimension $d_b$ & 512 & 640 & 512 & 512 \\
Compression dimension $d_s$ & 128 & 160 & 128 & 128 \\
Gradient clip (HiCI) & 0.3 & 0.3 & 0.3 & 0.3 \\

\midrule
\addlinespace[3pt]
\multicolumn{5}{l}{\textbf{Infrastructure}} \\
Precision & BF16 & BF16 & BF16 & BF16 \\
DeepSpeed & ZeRO Stage-2 & ZeRO Stage-2 & ZeRO Stage-2 & ZeRO Stage-2 \\
Attention kernel & Flash-Attention 2 & Flash-Attention 2 & Flash-Attention 2 & Flash-Attention 2 \\

\bottomrule
\end{tabular}
\end{table}

\subsection{Training Efficiency} 
\label{app:training_efficiency}
Figure~\ref{fig:training_efficiency} compares peak GPU memory usage and
wall-clock training time for HiCI and LongLoRA across context lengths
from 8K to 100K tokens (LLaMA-2-7B, 8$\times$H100-80GB, 1{,}000 steps;
DeepSpeed ZeRO Stage-2 for 8K--64K and Stage-3 for 100K).
\emph{In terms of memory}, HiCI introduces a modest overhead of
3.5--9.9\% relative to LongLoRA, arising from the learnable local and
global representations in the hierarchical pipeline. Since these
representations have fixed capacity per segment, the relative memory
gap narrows as context length increases and remains manageable even at
100K under ZeRO Stage-3.
\emph{In terms of wall-clock time}, while HiCI incurs at most 7.5\%
additional overhead at short contexts (8K--32K), it becomes
progressively faster at long contexts. At 100K tokens, HiCI adopts a
finer partitioning with $N{=}10$ segments of 10K tokens each, whereas
LongLoRA operates with $N{=}4$ segments of 25K tokens. Because
per-segment attention scales quadratically with segment length, this
finer-grained grouping substantially reduces the dominant attention
compute, yielding a 19.3\% reduction in total training time
(36.4\,h vs.\ 45.1\,h) despite the additional representational overhead.

Table~\ref{tab:flops} provides a per-layer FLOPs breakdown for Full Attention,
S$^2$-Attention, and HiCI across context lengths from 8K to 100K on LLaMA-2-7B. HiCI reduces total FLOPs by 16--60\% relative to Full Attention (e.g.,
585.0 vs.\ 996.0 TFLOPs at 32K; 2{,}762.6 vs.\ 6{,}850.7 at 100K), primarily
due to segmented attention over partitioned sequences. In addition, the Local
Construction and Global Integration (LC+GI) stages introduce only a 1--2\%
computational overhead compared to S$^2$-Attention (e.g., 585.0 vs.\ 573.7 at
32K; 2{,}762.6 vs.\ 2{,}727.5 at 100K), reflecting the cost of cross-segment
representation construction and aggregation.

\begin{table}[h]
\centering
\caption{Per-layer FLOPs (TFLOPs) for LLaMA-2-7B across context lengths.
LC+GI denotes the combined overhead of the Local Construction and Global
Integration stages; dashes indicate stages absent in that method.}
\label{tab:flops}
\small
\setlength{\tabcolsep}{4pt}

\begin{tabular}{cccccccc}
\toprule
\textbf{Context} & \textbf{Method} & \textbf{Attn} & \textbf{Proj} & \textbf{FFN} & \textbf{Others} & \textbf{LC+GI} & \textbf{Total} \\
\midrule

\multirow{3}{*}{8K}
& Full Attn   & 35.2   & 35.2  & 70.9  & 2.1  & --   & 143.4 \\
& S$^2$-Attn  & 8.8    & 35.2  & 70.9  & 2.1  & --   & 117.0 \\
& HiCI        & 9.4    & 35.2  & 70.9  & 2.1  & 2.2  & 119.9 \\
\midrule

\multirow{3}{*}{16K}
& Full Attn   & 140.7  & 70.4  & 141.8 & 4.3  & --   & 357.2 \\
& S$^2$-Attn  & 35.2   & 70.4  & 141.8 & 4.3  & --   & 251.7 \\
& HiCI        & 36.4   & 70.4  & 141.8 & 4.3  & 4.4  & 257.3 \\
\midrule

\multirow{3}{*}{32K}
& Full Attn   & 562.9  & 140.7 & 283.7 & 8.6  & --   & 996.0 \\
& S$^2$-Attn  & 140.7  & 140.7 & 283.7 & 8.6  & --   & 573.7 \\
& HiCI        & 143.1  & 140.7 & 283.7 & 8.6  & 8.8  & 585.0 \\
\midrule

\multirow{3}{*}{64K}
& Full Attn   & 2251.8 & 281.5 & 567.3 & 17.2 & --   & 3117.8 \\
& S$^2$-Attn  & 562.9  & 281.5 & 567.3 & 17.2 & --   & 1429.0 \\
& HiCI        & 567.8  & 281.5 & 567.3 & 17.2 & 17.6 & 1451.4 \\
\midrule

\multirow{3}{*}{100K}
& Full Attn   & 5497.6 & 439.8 & 886.5 & 26.8 & --   & 6850.7 \\
& S$^2$-Attn  & 1374.4 & 439.8 & 886.5 & 26.8 & --   & 2727.5 \\
& HiCI        & 1381.9 & 439.8 & 886.5 & 26.8 & 27.6 & 2762.6 \\
\bottomrule
\end{tabular}
\end{table}
 
\begin{figure}[t]       
\centering          
\includegraphics[width=0.8\textwidth]{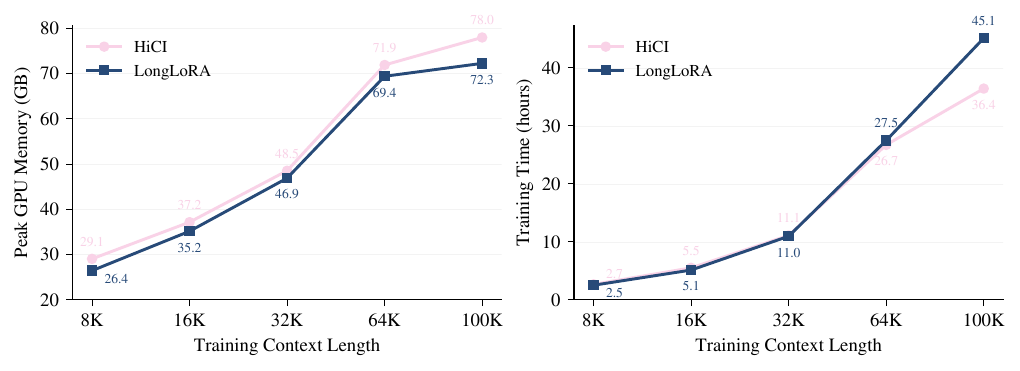} 
\caption{Peak GPU memory (left) and wall-clock training time (right) for                             
HiCI and LongLoRA (LLaMA-2-7B, 8$\times$H100-80GB, 1{,}000 steps;                                    
Stage-2 for 8K--64K, Stage-3 for 100K). The three-stage HiCI pipeline                                
raises memory by 3.5--9.9\%, which necessitates finer partitioning at                                
long contexts ($N{=}10$ at 100K vs.\ LongLoRA's $N{=}4$); the resulting                              
quadratic reduction in per-segment attention cost yields a 19.3\%                                    
wall-clock speedup.}                                                    
\label{fig:training_efficiency} 
\end{figure}   
   
\subsection{Training Loss Trajectories}
\label{app:training_loss}  
We compare the training dynamics of HiCI and LongLoRA during LLaMA-2-7B   
continual pre-training on RedPajama~\citep{together2023redpajama} over 
2{,}000 steps, varying context length (8K, 16K) and segment size     
($S \in \{1024, 2048\}$).
As shown in Figure~\ref{fig:training_loss},
HiCI with $S{=}1024$ exhibits sustained loss
reduction throughout training, with an additional decrease of 38\% over the
second half at 8K context length and 23\% at 16K. In contrast, HiCI with
$S{=}2048$ improves only marginally ($\sim$5\%), while all LongLoRA variants
plateau after approximately 1{,}000 steps ($\Delta{<}3\%$ thereafter).
The two methods display \emph{opposite preferences} with respect to segment
granularity. LongLoRA favors coarser segments (final loss 1.69 vs.\ 1.73 for
$S{=}2048$ vs.\ $1024$), consistent with prior findings that shifted sparse
attention benefits from wider per-head receptive fields~\citep{chen2024longlora}.
In contrast, HiCI improves substantially with finer segmentation (1.01 vs.\ 1.65),
suggesting that a larger number of segments yields richer local representations
for hierarchical aggregation.
Together, these results indicate that the two cross-segment mechanisms rely on
distinct inductive biases: direct attention over wider local windows versus
learnable compression and integration over more numerous segments.
\begin{figure}[t] 
\centering    
\includegraphics[width=\linewidth]{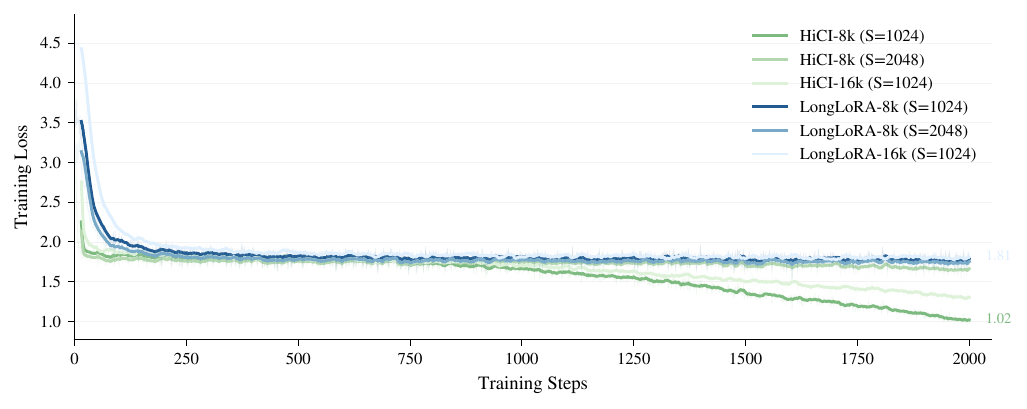} 
\caption{Training loss comparison between HiCI and LongLoRA on LLaMA-2-7B    
continual pre-training (RedPajama, 2{,}000 steps). Both methods are trained 
at 8K and 16K context with $S \in \{1024, 2048\}$. HiCI with $S{=}1024$   
sustains optimization throughout, while HiCI with $S{=}2048$ and all  
LongLoRA variants plateau beyond step 1{,}000.}                                
\label{fig:training_loss}  
\end{figure}   

\subsection{Parameter Overhead}
\label{app:params}

HiCI introduces additional learnable parameters that are independent of
sequence length. Table~\ref{tab:params-appendix} provides a detailed breakdown for
LLaMA-2-7B with 32 transformer layers.
\begin{table}[h]
\centering
\caption{Parameter overhead of HiCI on LLaMA-2-7B. We use $d{=}4096$, bottleneck dimension $d_b{=}512$,
shared compression dimension $d_s{=}128$, $M{=}8$ local slots, and $K{=}4$ global slots across 32 layers.}
\label{tab:params-appendix}
\small
\begin{tabular}{llrr}
\toprule
\textbf{Module} & \textbf{Component} & \textbf{Per Layer} & \textbf{Total (32L)} \\
\midrule
\multirow{3}{*}{\textit{Local Construction}}
  & Memory slots ($M{=}8$) & 32.8K & 1.0M \\
  & Cross-attention (Q/K/V/O) & 8.4M & 268.4M \\
\cmidrule{2-4}
  & \textbf{Subtotal} & \textbf{8.4M} & \textbf{269.5M} \\
\midrule
\multirow{5}{*}{\textit{Global Integration}}
  & Shared compression ($d_s{=}128$) & 591.1K & 18.9M \\
  & Global queries ($K{=}4$) & 2.0K & 0.1M \\
  & Lightweight attention (Q/K/V/O) & 1.0M & 33.6M \\
  & Expansion layer & 2.1M & 67.1M \\
\cmidrule{2-4}
  & \textbf{Subtotal} & \textbf{3.7M} & \textbf{119.6M} \\
\midrule
\multicolumn{2}{l}{\textbf{HiCI Total}} & \textbf{12.2M} & \textbf{389.1M} \\
\multicolumn{2}{l}{Base Model (LLaMA-2-7B)} & --- & 6.74B \\
\multicolumn{2}{l}{\textbf{Parameter Overhead}} & --- & \textbf{5.46\%} \\
\bottomrule
\end{tabular}
\end{table}

The parameter overhead is modest (5.46\%) relative to the base model and,
importantly, does not scale with sequence length---the same parameters
handle 4K, 32K, or 100K contexts without modification.

\section{Layer-wise Attention Analysis}
\label{app:attention_analysis}

We analyze how HiCI routes attention across hierarchical representations by recording layer-wise attention statistics during evaluation on PG-19. For each layer, we compute the fraction of total attention mass assigned to the $K{=}4$ global slots, averaged over all heads and evaluation samples. At each layer, the key--value sequence consists of $K$ global slots, $M{=}8$ local slots, and $S$ segment tokens, yielding a total length of $K{+}M{+}S$ (1036 for $S{=}1024$; 2060 for $S{=}2048$). Fig.~\ref{fig:layer_attention}(a) compares $S{=}1024$ and $S{=}2048$ under matched conditions (8K evaluation, 2K training steps), while Fig.~\ref{fig:layer_attention}(b) probes robustness at $S{=}2048$ by varying evaluation length and training duration.

\paragraph{Depth-dependent routing.}
Across configurations, attention to global slots exhibits a clear increasing trend with layer depth, despite minor layer-wise variations. Averaged over early layers (L0--7), global attention ranges from 1\% to 8\% across configurations; for deep layers (L24--31), it ranges from 6\% to 26\%, yielding deep-to-early ratios of 3.3--4.9$\times$. At the final layer (L31), global attention reaches 40.4\% for $S{=}1024$ ($\approx$105$\times$ the uniform baseline of 0.39\%) and 12.7\% for $S{=}2048$ ($\approx$65$\times$ the baseline of 0.19\%). As no explicit supervision is imposed on attention allocation, this stratification suggests that deeper layers allocate attention preferentially to hierarchical context.

\paragraph{Effect of segment granularity.}
Reducing $S$ from 2048 to 1024 amplifies global attention by approximately 4--7$\times$ across all depth groups, substantially exceeding the $2\times$ change in the proportional presence of global slots ($\tfrac{4}{1036}$ vs.\ $\tfrac{4}{2060}$). This behavior is consistent with the divergent scaling observed in Section~\ref{sec:ablation}: finer segmentation tightens the local information bottleneck and is associated with stronger reliance on global aggregation.

\paragraph{Robustness.}
At fixed $S{=}2048$, layer-wise allocation patterns remain largely stable when the evaluation length is halved (8K$\to$4K) or when training is shortened (2K$\to$1K steps), with per-layer deviations within 1 percentage point for most layers. This suggests that hierarchical routing emerges early in training and generalizes across sequence lengths.

\begin{figure}[t]
  \centering
  \includegraphics[width=0.6\textwidth]{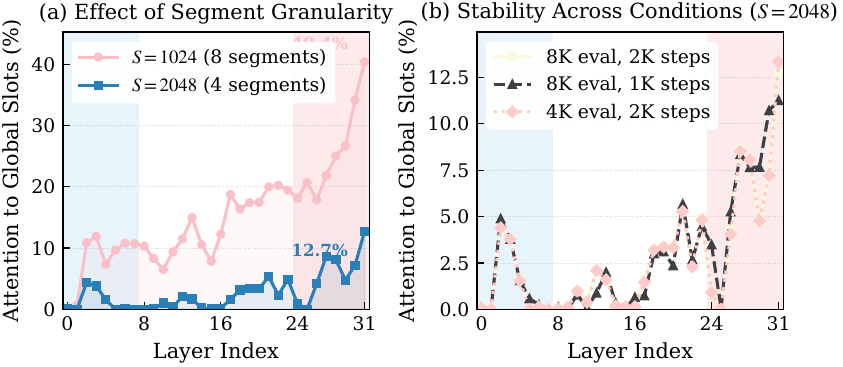}
  \caption{%
  Layer-wise attention allocated to global slots during evaluation on PG-19.
  Background shading denotes depth groups: early (L0--7, blue), middle (L8--23, white), and deep (L24--31, red).
  \textbf{(a)}~Comparison of segment sizes $S{=}1024$ and $S{=}2048$ under matched conditions (8K evaluation, 2K steps): finer segmentation yields substantially higher attention to global slots, with the final layer (L31) reaching 40.4\% versus 12.7\%.
  \textbf{(b)}~Robustness at $S{=}2048$: varying evaluation length (8K vs.\ 4K) and training steps (2K vs.\ 1K) yields nearly identical layer-wise allocation patterns, with per-layer deviations within 1 percentage point.
  In both panels, attention to global slots increases toward deeper layers.%
  }
  \label{fig:layer_attention}
\end{figure}

\end{document}